\documentclass[11pt]{article}

\usepackage{fullpage}

\usepackage[utf8]{inputenc} 
\usepackage[T1]{fontenc}    

\usepackage{amsmath,amsthm,amssymb,amsfonts}

\usepackage[round]{natbib}

\usepackage{pgfplots}
\pgfplotsset{compat=1.15}

\usepackage{tikz}
\usepackage{graphicx}
\usepackage{caption}

\usepackage{algorithm}
\usepackage{algorithmic}

\PassOptionsToPackage{linktocpage}{hyperref}
\usepackage{xr-hyper}
\usepackage[hyperindex,breaklinks]{hyperref}    

\usepackage{mathtools}
\usepackage{booktabs}
\usepackage{microtype}

\usepackage{xfrac}
\usepackage{nicefrac}    

\usepackage{dsfont}
\usepackage[T1]{fontenc}
\usepackage{lmodern}

\usepackage{authblk} 
\newtheorem{theorem}{Theorem}[section]
\newtheorem{lemma}[theorem]{Lemma}
\newtheorem{claim}[theorem]{Claim}

\newtheorem{corollary}[theorem]{Corollary}

\DeclarePairedDelimiter\ceil{\lceil}{\rceil}
\DeclarePairedDelimiter\floor{\lfloor}{\rfloor}

\def\expect#1{\mathbb{E}\left[#1\right]}
\def\prob#1{\mathbb{P}\left[#1\right]}
\def\pr{\mathbb{P}}
\def\abs#1{\left\lvert#1\right\rvert}

\newcommand{\Pois}{\mathrm{Pois}}

\def\rbr#1{\left(#1\right)}   
\def\sbr#1{\left[#1\right]}   
\def\cbr#1{\left\{#1\right\}}  

\def\norm#1{\lVert#1\rVert}

\newcommand{\R}{\mathbb{R}}

\newcommand{\real}{\mathbb{R}}

\newcommand{\de}{\mathrm{d}}

\newcommand{\opt}{\mathrm{OPT}}
\newcommand{\cost}{\mathrm{cost}}
\newcommand{\cluster}{P}
\newcommand{\total}{\mathbf{X}}
\newcommand{\Y}{\mathbf{Y}}
\newcommand{\centerset}{C}

\newcommand{\potential}{\Psi}
\newcommand{\dimension}{d}
\newcommand{\pointset}{\mathbf{X}}
\newcommand{\extracenters}{\Delta}

\newcommand{\calE}{\mathcal{E}}
\newcommand{\calD}{\mathcal{D}}
\newcommand{\MM}{\partial \calM}

\newcommand{\ip}[1]{\left\langle #1 \right\rangle}

\newcommand{\one}{\mathds{1}}
\newcommand{\kpp}{\text{$\parallel_{\Pois}$}}
\newcommand{\kpois}{\text{$++_{\text{ER}}$}}

\newcommand{\calP}{{\cal{P}}}
\newcommand{\calM}{{\cal{M}}}
\newcommand{\calH}{{\cal{H}}}
\newcommand{\E}{\mathbb{E}}
\newcommand{\OPT}{\operatorname{OPT}}
\newcommand{\HH}{\widetilde{H}}

\renewcommand{\cite}[1]{\citep{#1}}

\ifdefined\nohyperref\else\ifdefined\hypersetup
  \definecolor{mydarkblue}{rgb}{0,0.08,0.45}
  \hypersetup{ %
    colorlinks=true,
    linkcolor=mydarkblue,
    citecolor=mydarkblue,
    filecolor=mydarkblue,
    urlcolor=mydarkblue,
    pdfview=FitH}
\fi\fi

\renewcommand\footnotemark{}

\title{Improved Guarantees for $k$-means++ and $k$-means++ Parallel\footnote{The conference version of this paper will appear in the proceedings of the 34th Conference on Neural Information Processing Systems (NeurIPS 2020). Author order is alphabetical.}}

\author{Konstantin Makarychev}
\author{Aravind Reddy}
\author{Liren Shan}
\affil{Department of Computer Science
    \\Northwestern University
    \\Evanston, IL, USA}

\begin{document}
\date{}
\maketitle

\begin{abstract}
In this paper, we study $k$-means++ and $k$-means$\parallel$,  the two most popular algorithms for the classic $k$-means clustering problem. We provide novel analyses and show improved approximation and bi-criteria approximation guarantees for $k$-means++ and $k$-means$\parallel$. Our results give a better theoretical justification for why these algorithms  perform extremely well in practice. We also propose a new variant of $k$-means$\parallel$ algorithm (Exponential Race $k$-means++) that has the same approximation guarantees as $k$-means++.
\end{abstract}
\section{Introduction}

$k$-means clustering is one of the most commonly encountered unsupervised learning problems. Given a set of $n$ data points in Euclidean space, our goal is to partition them into $k$ clusters (each characterized by a center), such that the sum of squares of distances of data points to their nearest centers is minimized.
The most popular heuristic for solving this problem is Lloyd's algorithm \cite{lloyd}, often referred to simply as ``the $k$-means algorithm".

Lloyd's algorithm uses iterative improvements to find a locally optimal $k$-means clustering.  The performance of Lloyd's algorithm crucially depends on the quality of the initial clustering, which is defined by the initial set of centers, called a \emph{seed}. \citet*{arthur2007k} and \citet*{ostrovsky2006effectiveness} developed an elegant randomized seeding algorithm, known as the $k$-means++ algorithm. It works by choosing the first center uniformly at random from the data set and then choosing the subsequent $k-1$ centers by randomly sampling a single point in each round with the sampling probability of every point proportional to its current cost. That is, the probability of choosing any data point $x$ is proportional to the squared distance to its closest already chosen center. This squared distance is often denoted by $D^2(x)$. \citet*{arthur2007k} proved that the expected cost of the initial clustering obtained by $k$-means++ is at most $8\rbr{\ln k + 2}$ times the cost of the optimal clustering i.e., $k$-means++ gives an $8\rbr{\ln k + 2}$-approximation for the $k$-means problem. They also provided a family of $k$-means instances for which the approximation factor of $k$-means++ is $2\ln k$ and thus showed that
their analysis of $k$-means++ is almost tight.

Due to its speed, simplicity, and good empirical performance, $k$-means++ is the most widely used algorithm for $k$-means clustering. It is employed by such machine learning libraries as Apache Spark MLlib, Google BigQuery, IBM SPSS, Intel DAAL, and Microsoft ML.NET. In addition to $k$-means++, these libraries implement a scalable variant of $k$-means++ called $k$-means$\parallel$ (read ``$k$-means parallel'') designed by \citet*{bahmani2012scalable}. Somewhat surprisingly, $k$-means$\parallel$ not only works better in parallel than $k$-means++ but also slightly outperforms $k$-means++ in practice in the single machine setting (see~\citet{bahmani2012scalable} and Figure~\ref{fig:experiment} below).
However, theoretical guarantees for $k$-means$\parallel$ are substantially weaker than for $k$-means++.

The $k$-means$\parallel$ algorithm makes $T$ passes over the data set (usually $T = 5$). In every round, it independently draws approximately $\ell = \Theta(k)$ random centers according to the $D^2$ distribution. After each round it recomputes the distances to the closest chosen centers and updates $D^2(x)$ for all $x$ in the data set. Thus, after $T$ rounds, $k$-means$\parallel$ chooses approximately $T\ell$ centers. It then selects $k$ centers among $T\ell$ centers using $k$-means++.

\medskip

\noindent\textbf{Our contributions. } In this paper, we improve the theoretical guarantees for $k$-means++, $k$-means$\parallel$, and Bi-Criteria $k$-means++ (which we define below).

First, we show that the expected cost of the solution output by $k$-means++ is at most $5(\ln k +2)$ times the optimal solution's cost. This improves upon the bound of $8(\ln k +2)$ shown by \citet*{arthur2007k} and directly improves the approximation factors for several algorithms which use $k$-means++ as a subroutine like Local Search k-means++ \citep*{lattanzi2019better}. To obtain this result, we give a refined analysis of the expected cost of \emph{covered clusters} (see Lemma 3.2 in \citet*{arthur2007k} and Lemma~\ref{lem:5OPT} in this paper). We also show that our new bound on the expected cost of \emph{covered clusters} is tight (see Lemma~\ref{thm:5-approx-tight}).

Then, we address the question of why the observed performance of $k$-means$\parallel$ is better than the performance of $k$-means++. There are two possible explanations for this fact. (1) This may be the case because $k$-means$\parallel$ picks $k$ centers in two stages. At the first stage, it samples $\ell T \geq k$ centers. At the second stage, it prunes centers and chooses $k$ centers among $\ell T$ centers using $k$-means++.  (2) This may also be the case because $k$-means$\parallel$ updates the distribution function  $D^2(x)$ once in every round. That is, it recomputes $D^2(x)$ once for every $\ell$ chosen centers, while $k$-means++ recomputes $D^2(x)$ every time it chooses a center. In this paper, we empirically demonstrate that the first explanation is correct.
First, we noticed that $k$-means$\parallel$ for $\ell\cdot T = k$ is almost identical with $k$-means++ (see Appendix~\ref{sec:experiments}).
Second, we compare $k$-means$\parallel$ with another algorithm which we call Bi-Criteria $k$-means++ with Pruning. This algorithm also works in two stages: At the Bi-Criteria $k$-means++ stage, it chooses $k+\Delta$ centers in the data set using $k$-means++. Then, at the Pruning stage, it picks $k$ centers among the $k+\Delta$ centers selected at the first stage again using $k$-means++. Our experiments on the standard data sets BioTest from KDD-Cup 2004 \cite{kddcup2004} and COVTYPE from the UCI ML repository \cite{Dua:2019} show that the performance of $k$-means$\parallel$ and Bi-Criteria $k$-means++ with Pruning are essentially identical (see Figures~\ref{fig:experiment} and Appendix~\ref{sec:experiments}).

\begin{figure}
    \begin{minipage}[h]{0.45\linewidth}
        \begin{tikzpicture}[scale=0.7]
        \begin{axis}[
        title= {(a) BioTest},
        xlabel={\#centers},
        ylabel={cost},
        grid = major]
        \addplot[black,domain=1:2, line width = 1pt]  table[x=centers,y=avgKMeansPP,col sep=comma] {plotdata/bio-test10-50.csv}; \addlegendentry{$k$-means++}
        \addplot[red,domain=1:2, line width = 0.5pt]  table[x=centers,y=avgBicriteria,col sep=comma] {plotdata/bio-test10-50.csv}; \addlegendentry{BiCriteria $k$-means++ w/Pruning}
        \addplot [blue, domain=1:2, dashed, line width = 1pt] table[x=centers,y=avgKMeansParallel,col sep=comma] {plotdata/bio-test10-50.csv}; \addlegendentry{$k$-means$\parallel$}
        \end{axis}
        \end{tikzpicture}
        \captionsetup{justification=centering}
    \end{minipage}
    \quad
    \begin{minipage}[h]{0.45\linewidth}
        \begin{tikzpicture}[scale=0.7]
        \begin{axis}[
        title= {(b) COVTYPE},
        xlabel={\#centers},
        ylabel={cost},
        grid = major]
        \addplot[black,domain=1:2, line width = 1pt]  table[x=centers,y=avgKMeansPP,col sep=comma] {plotdata/covtype10-50.csv}; \addlegendentry{$k$-means++}
        \addplot[red,domain=1:2, line width = 0.5pt]  table[x=centers,y=avgBicriteria,col sep=comma] {plotdata/covtype10-50.csv}; \addlegendentry{BiCriteria $k$-means++ w/Pruning}
        \addplot [blue, domain=1:2, dashed, line width = 1pt] table[x=centers,y=avgKMeansParallel,col sep=comma] {plotdata/covtype10-50.csv};\addlegendentry{$k$-means$\parallel$}
        \end{axis}
        \end{tikzpicture}
        \captionsetup{justification=centering}
    \end{minipage}
    \caption{Performance of $k$-means++, $k$-means$\parallel$, and Bi-Criteria  $k$-means++ with pruning on the BioTest and COVTYPE datasets. For $k=10,15,\cdots, 50$, we ran these algorithms for 50 iterations and took their average. We normalized the clustering costs by dividing them by $\cost_{1000}(\pointset)$.}
    \label{fig:experiment}
\end{figure}
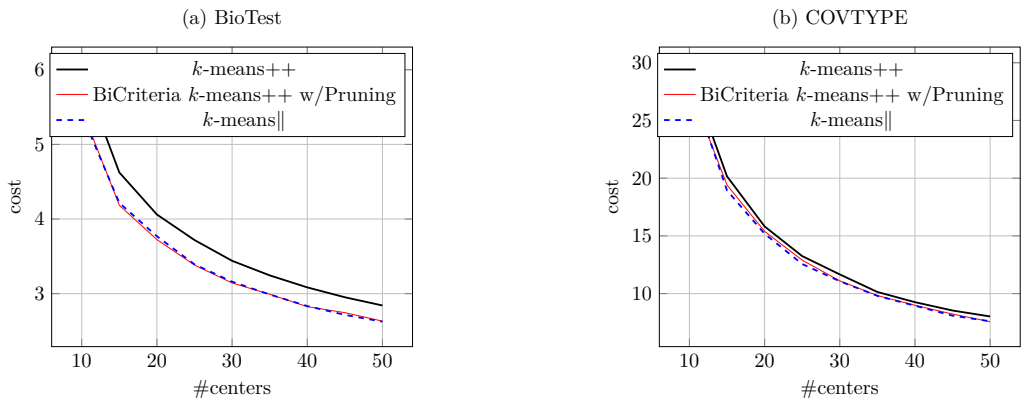

These results lead to another interesting question: How good are $k$-means++ and $k$-means$\parallel$ algorithms that sample $k+\Delta$ instead of $k$ centers? The idea of oversampling using $k$-means++ was studied earlier in the literature under the name of \emph{bi-criteria approximation}. ~\citet*{aggarwal2009adaptive} showed that with constant probability, sampling $k + \Delta$ centers by $k$-means++ provides a constant-factor approximation if $\Delta \geq \beta k$ for some constant $\beta > 0$. \citet{wei2016constant} improved on this result by showing an expected approximation ratio of $8(1+1.618 k/\extracenters)$. Note that for bi-criteria algorithms we compare the expected cost of the clustering with $k+\Delta$ centers they produce and the cost of the optimal clustering with exactly $k$ centers.

In this paper, we show that the expected bi-criteria approximation ratio for $k$-means++ with $\Delta$ additional centers is at most the minimum of two bounds:
$$\text{(A) } 5\rbr{2 + \frac{1}{2e} + \ln{\frac{2k}{\extracenters}}} \text{ for } 1\leq \Delta \leq 2k;\text{ and (B) }5\rbr{1+ \frac{k}{e\rbr{\extracenters-1}}} \text{ for }\Delta \geq 1$$
Both bounds are better than the bound by~\citet{wei2016constant}. The improvement is especially noticeable for small values of $\Delta$. More specifically, when the number of additional centers is $\Delta = k/ \log k$, our approximation guarantee is $O(\log\log k)$ while  \citet{wei2016constant} gives an $O(\log k)$ approximation.

We believe that our results for small values of $\Delta$ provide an additional explanation for why $k$-means++ works so well in practice. Consider a data scientist who wants to cluster a data set $\pointset$ with $k^*$ \emph{true clusters} (i.e. $k^*$ latent groups). Since she does not know the actual value of $k^*$, she uses the \emph{elbow method} \cite{boehmke2019hands} or some other heuristic to find $k$. Our results indicate that if she chooses slightly more number of clusters (for instance, $1.05 k^*$), then she will get a constant bi-criteria approximation to the optimal clustering.

We also note that our bounds on the approximation factor smoothly transition from the regular ($\Delta = 0$) to bi-criteria ($\Delta > 0$) regime. We complement our analysis with an almost matching lower bound of $\Theta(\log (k/\Delta))$ on the approximation factor of $k$-means for $\Delta \leq k$ (see Appendix~\ref{sec:lb}).

We then analyze Bi-Criteria $k$-means$\parallel$ algorithm, the variant of $k$-means$\parallel$ that does not prune centers at the second stage. In their original paper, \citet*{bahmani2012scalable} showed that the expected cost of the solution for $k$-means$\parallel$ with $T$ rounds and oversampling parameter $\ell$ is at most:
$$\frac{16}{1 - \alpha} \opt_k(\pointset) + \Big(\frac{1+\alpha}{2}\Big)^T \opt_1(\pointset),$$
where $\alpha = \exp\rbr{-\rbr{1-e^{-\ell/(2k)}}} $;
$\opt_k(\pointset)$ is the cost of the optimal $k$-means clustering of $\pointset$; $\opt_1(\pointset)$ is the cost of the optimal clustering of $X$ with 1 center (see Section~\ref{sec:prelim} for details). We note that $\opt_1(\pointset) \gg \opt_k(\pointset)$. For $\ell = k$, this result gives a bound of $\approx 49\, \opt_k(\pointset) + 0.83^T\opt_1(\pointset)$.
\citet*{bachem2017distributed} improved the approximation guarantee for $\ell \geq k$ to $$26 \opt_k(\pointset) + 2\Big(\frac{k}{e\ell}\Big)^T\opt_1(\pointset).$$
In this work, we improve this bound for $\ell \geq k$ and also obtain a better bound for $\ell < k$. For $\ell \geq k$, we show that the cost of $k$-means$\parallel$
without pruning is at most
$$8 \opt_k(\pointset) + 2\Big(\frac{k}{e\ell}\Big)^T\opt_1(\pointset).$$
For $\ell < k$, we give a bound of
$$
\frac{5}{1-e^{-\frac{\ell}{k}}}\;\opt_k(\pointset)+
2\rbr{e^{-\frac{\ell}{k}}}^T \opt_1(\pointset)$$

Finally, we give a new parallel variant of the $k$-means++ algorithm, which we call \emph{Exponential Race $k$-means++} ($k$-means$\kpois$). This algorithm is similar to $k$-means$\parallel$.
In each round, it also selects $\ell$ candidate centers in parallel (some of which may be dropped later) making one pass over the data set. However, after $T$ rounds, it returns exactly $k$ centers. The probability distribution of these centers is identical to the distribution of centers output by $k$-means++. The expected number of rounds is bounded as follows:
$$O\bigg(\frac{k}{\ell}+ \log\frac{\OPT_1(\pointset)}{\OPT_k(\pointset)}\bigg).$$
This algorithm offers a unifying view on $k$-means++ and $k$-means$\parallel$.
We describe it in Section~\ref{sec:pois-kmeans-pp}.

\medskip

\noindent\textbf{Other related work.} \citet*{dasgupta2008hardness} and \citet*{aloise2009np} showed that $k$-means problem is NP-hard. \citet*{awasthi2015hardness} proved that it is also NP-hard to approximate $k$-means objective within a factor of $(1+\varepsilon)$ for some constant $\varepsilon>0$ (see also \citet*{LSW17}). We also mention that $k$-means was studied not only for Euclidean spaces but also for arbitrary metric spaces.

There are several known \emph{constant} factor approximation algorithms for the $k$-means problem. \citet*{kanungo2004local} gave a $9+\varepsilon$ approximation local search algorithm. \citet*{ahmadian2019better} proposed a primal-dual algorithm with an approximation factor of $6.357$. This is the best known approximation for $k$-means. \citet*{makarychev2016bi} gave constant-factor bi-criteria approximation algorithms based on linear programming and local search. Note that although these algorithms run in polynomial time, they do not scale well to massive data sets. \citet*{lattanzi2019better} provided a constant factor approximation by combining the local search idea with the $k$-means++ algorithm. \citet*{choo2020k} further improved upon this result by reducing the number of local search steps needed from $O(k\log\log k)$ to $O(k)$.

Independently and concurrently to our work, \citet*{rozhovn2020simple} gave an interesting analysis for $k$-means$\parallel$ by viewing it as a \textit{balls into bins} problem and showed that $O(\log n / \log \log n)$ rounds suffice to give a constant approximation with high probability.

\medskip

\noindent\textbf{Acknowledgments.}
We would like to thank all the reviewers for their helpful comments. Konstantin Makarychev, Aravind Reddy, and Liren Shan were supported in part by NSF grants CCF-1955351 and HDR TRIPODS CCF-1934931. Aravind Reddy was also supported in part by NSF CCF-1637585.

\section{Preliminaries}\label{sec:prelim}
Given a set of points $\pointset = \cbr{x_1,x_2,\cdots, x_n} \subseteq \real^\dimension$ and an integer $k \geq 1$, the $k$-means clustering problem is to find a set $\centerset$ of $k$ centers in $\real^\dimension$ to minimize
$$
    \cost(\pointset,C) \coloneqq \sum_{x \in \pointset} \min_{c\in \centerset} \norm{x-c}^2.
$$
For any integer $i \geq 1$, let us define $ \opt_i(\pointset) \coloneqq \min_{\abs{\centerset} = i} \cost\rbr{\pointset, \centerset}.$ Thus, $\opt_k(\pointset)$ refers to the cost of the optimal solution for the $k$-means problem. Let $\centerset^*$ denote a set of optimal centers. We use $\cbr{\cluster_i}_{i=1}^k$ to denote the clusters induced by the center set $\centerset^*$.

For any $\Y \subseteq \pointset$, the cost of $\Y$ with center set $C$, denoted by $\cost\rbr{\mathbf{Y},\centerset} = \sum_{x \in \mathbf{Y}} \min_{c\in \centerset} \norm{x-c}^2$. The optimal cost for subset $\Y$ with $i$ centers is $\opt_i(\Y)$.
Let $\mu =\sum_{x \in \Y} x / \abs{\Y}$ be the \textit{centroid} of the cluster $\Y$.  Then, we have the following closed form expression for the optimal cost of $\Y$ with one center (see Appendix~\ref{sec:prelim_details} for proof),
\begin{align}\label{eq:opt-1-closed-form}
    \opt_1(\Y) = \sum_{x\in \Y} \norm{x-\mu}^2 = \frac{\sum_{(x,y)\in\Y\times\Y} \norm{x-y}^2}{2\abs{\Y}}.
\end{align}

\textbf{$k$-means++ seeding:}
The $k$-means++ algorithm samples the first center uniformly at random from the given points and then samples $k-1$ centers sequentially from the given points with probability of each point being sampled proportional to its cost i.e. $\cost(x,C)/ \cost(\pointset,C)$.

\begin{algorithm}
   \caption{$k$-means++ seeding}
   \label{alg:kmeans++}
\begin{algorithmic}[1]
   \STATE Sample a point $c$ uniformly at random from $\pointset$ and set $\centerset_1 = \cbr{c}$.
   \FOR{$t=2$ {\bfseries to} $k$}
   \STATE Sample $x \in \pointset$ w.p. $\cost(x,C_t)/\cost(\pointset,C_t)$.
   \STATE $\centerset_t = \centerset_{t-1} \cup \{x\}$.
   \ENDFOR
   \STATE \textbf{Return} $C_k$
\end{algorithmic}
\end{algorithm}

\textbf{$k$-means$\parallel$ and $k$-means$\kpp$ seeding:}
In the $k$-means$\parallel$ algorithm, the first center is chosen uniformly at random from $\pointset$. But after that, at each round, the algorithm samples each point independently with probability $\min\cbr{\ell\cdot\cost(x,C)/\cost(\pointset,C),1}$ where $\ell$ is the \textit{oversampling parameter} chosen by the user and it usually lies between $0.1k$ and $10k$. The algorithm runs for $T$ rounds (where $T$ is also a parameter chosen by the user) and samples around $\ell T$ points, which is usually strictly larger than $k$. This oversampled set is then weighted using the original data set $\pointset$ and a weighted version of $k$-means++ is run on this set to get the final $k$-centers. We only focus on the stage in which we get the oversampled set because the guarantees for the second stage come directly from $k$-means++.

For the sake of analysis, we also consider a different implementation of $k$-means$\parallel$, which we call $k$-means\kpp\; (Algorithm~\ref{alg:kmeanpp_pd}). This algorithm differs from $k$-means$\parallel$ in that each point is sampled independently with probability $1-\exp(-\ell\cdot\cost(x,C)/\cost(\pointset,C))$ rather than $\min\{\ell\cdot\cost(x,C)/\cost(\pointset,C),1\}$. In practice, there is essentially no difference between $k$-means$\parallel$ and $k$-means$\kpp$, since $\ell\cdot\cost(x,C)/\cost(\pointset,C)$ is a very small number for all $x$ and thus the sampling probabilities for $k$-means$\parallel$ and $k$-means$\kpp$ are almost equal.

\begin{minipage}{0.45\textwidth}
\begin{algorithm}[H]
\caption{$k$-means$\parallel$ seeding}
\label{alg:kmeanpp}
\begin{algorithmic}[1]
\STATE Sample a point $c$ uniformly from $\total$ and set $\centerset_1 = \cbr{c}$
\FOR{$t=1$ {\bfseries to} $T$}
\STATE Sample each point $x$ into $\centerset^\prime$ independently w.p. $\min\{1,\lambda_{t}(x)\}$ where \\$\lambda_{t}(x) = \ell\cdot\cost(x,C_t)/ \cost(\total,C_t)$
\STATE Let $\centerset_{t+1} = \centerset_{t} \cup \centerset^\prime$.
\ENDFOR
\end{algorithmic}
\end{algorithm}
\end{minipage}
\quad
\begin{minipage}{0.45\textwidth}
\begin{algorithm}[H]
\caption{$k$-means$\kpp$ seeding}
\label{alg:kmeanpp_pd}
\begin{algorithmic}[1]
\STATE Sample a point $c$ uniformly from $\total$ and set $\centerset_1 = \cbr{c}$
\FOR{$t=1$ {\bfseries to} $T$}

\STATE Sample each point $x$ into $\centerset^\prime$ independently w.p. $1-e^{-\lambda_{t}(x)}$ where \\$\lambda_{t}(x) = \ell\cdot\cost(x,C_t)/ \cost(\total,C_t)$
\STATE Let $\centerset_{t+1} = \centerset_{t} \cup \centerset^\prime$.
\ENDFOR
\end{algorithmic}
\end{algorithm}
\end{minipage}

In the rest of the paper, we focus only on the \emph{seeding} step of $k$-means++, $k$-means$\parallel$, and $k$-means$\kpp$ and ignore Lloyd's iterations as the approximation guarantees for these algorithms come entirely from the seeding step.

\section{General framework}\label{sec:framework}
In this section, we describe a general framework we use to analyze $k$-means++ and $k$-means\kpp. Consider $k$-means++ or $k$-means$\kpp$ algorithm.
Let $C_t$ be the set of centers chosen by this algorithm after step $t$. For the sake of analysis, we assume that $C_t$ is an ordered set or list of centers, and the order of centers in $C_t$ is the same as the order in which our algorithm chooses these centers.
We explain how to order centers in
$k$-means$\kpp$ algorithm in Section~\ref{sec:po-kmeans-parallel}.
We denote by $T$ the stopping time of the algorithm.
Observe that after step $t$ of the algorithm, the probabilities of choosing a new center in $k$-means++ or a batch of new centers in $k$-means$\kpp$ are defined by the current costs of points in $\pointset$ which, in turn, are completely determined by the current set of centers $\centerset_t$. Thus, the states of the algorithm form a Markov chain.

In our analysis, we fix the optimal clustering $\calP = \{\cluster_1,\dots, \cluster_k\}$ (if this clustering is not unique, we pick an arbitrary optimal clustering). The optimal cost of each cluster $\cluster_i$ is $\opt_1(\cluster_i)$ and the optimal cost of the entire clustering is
$\opt_k(\pointset)=\sum_{i=1}^k \opt_1(\cluster_i)$.

Following the notation in \citet*{arthur2007k}, we say that a cluster $\cluster_i$ is \emph{hit} or \emph{covered} by a set of centers $C$ if $C\cap P_i \neq \varnothing$; otherwise, we say that $P_i$ is \emph{not hit} or \emph{uncovered}. We split the cost of each cluster $P_i$ into two
components which we call the covered and uncovered costs of $P_i$. For a given set of centers $\centerset$,
\begin{align*}
    \text{The covered or hit cost of } P_i,\qquad H(P_i,C) &\coloneqq
\begin{cases}
  \cost(P_i,C), & \mbox{if $P_i$ is covered by $C$ } \\
  0, & \mbox{otherwise}.
\end{cases}
\\\text{The uncovered cost of } P_i,\qquad U(P_i,C) &\coloneqq
\begin{cases}
  0, & \mbox{if $P_i$ is covered by $C$ } \\
  \cost(P_i,C), & \mbox{otherwise}.
\end{cases}
\end{align*}
Let $H(\pointset,\centerset) = \sum_{i=1}^k H(\cluster_i,\centerset)$ and $U(\pointset,\centerset) = \sum_{i=1}^k U(\cluster_i,\centerset)$. Then,
$$\cost(\pointset,C) = H(\pointset,C) + U(\pointset,C).$$
For the sake of brevity, we define $\cost_t(\Y) \coloneqq \cost(\Y,C_t)$ for any $\Y \subseteq \pointset$, $H_t(P_i) \coloneqq H(P_i,C_t)$, and $U_t(P_i) \coloneqq U(P_i,C_t)$.
In Section~\ref{sec:5OPT}, we show that for any $t$, we have $\E[H_t(\pointset)]\leq 5\opt_k(\pointset)$,
which is an improvement over the bound of $8\opt_k(\pointset)$ given by \citet*{arthur2007k}.
Then, in Sections~\ref{sec:po-bi-criteria} and \ref{sec:po-kmeans-parallel}, we analyze the expected uncovered cost $U(\pointset,C_T)$ for $k$-means++ and $k$-means$\parallel$ algorithms.

Consider a center $c$ in $\centerset$. We say that $c$ is a \textit{miss} if another center $c'$ covers the same cluster in $\calP$ as $c$, and $c'$ appears before $c$ in the ordered set $C$.
We denote the number of misses in $C$ by $M(C)$ and the the number of clusters in $\calP$ not covered by centers in $C$ by $K(C)$.

Observe that the stochastic processes $U_t(P_i)$ with discrete time $t$ are non-increasing since the algorithm never removes centers from the set $C_t$ and therefore the distance from any point $x$ to $C_t$ never increases. Similarly, the processes $H_t(P_i)$ are non-increasing after the step $t_i$ when $P_i$ is covered first time.
In this paper, we sometimes use a proxy $\HH_t(P_i)$ for $H_t(P_i)$, which we define as follows.
If $P_i$ is covered by $C_t$, then $\HH_t(P_i) = H_{t_i}(P_i)$, where $t_i\leq t$ is the first time when $P_i$ is covered by $C_t$.
If $P_i$ is not covered by $C_t$, then $\HH_t(P_i)=5\opt_1(P_i)$. It is easy to see that $H_{t}(P_i) \leq \HH_{t'}(P_i)$ for all
$t\leq t'$.
In Section~\ref{sec:5OPT}, we also show that $\HH_t(P_i)$ is a supermartingale i.e., $\E [\HH_{t'}(P_i) \mid C_{t}] \leq \HH_{t}(P_i)$ for all $t\leq t'$.

\section{Bound on the cost of covered clusters}\label{sec:5OPT}
In this section, we improve the bound by \citet*{arthur2007k} on the expected cost of a covered cluster in $k$-means++. Our bound also works for $k$-means$\kpp$ algorithm. Pick an arbitrary cluster $P_i$ in the optimal solution $\calP= \cbr{P_1,\dots, P_k}$ and consider an arbitrary state
$C_t = \cbr{c_1,\dots,c_N}$ of the $k$-means++ or $k$-means$\parallel_{\Pois}$ algorithm. Let $D_{t+1}$ be the set of new centers the algorithm adds to $C_t$
at step $t$ (for $k$-means++, $D_{t+1}$ contains only one center). Suppose now that centers in $D_{t+1}$ cover $P_i$ i.e. $D_{t+1}\cap P_i\neq \varnothing$.
We show that the expected cost of cluster $P_i$ after step $(t+1)$ conditioned on the event $\{D_{t+1}\cap P_i\neq \varnothing\}$ and the current state of the algorithm $C_t$ is upper bounded by $5\opt_1(P_i)$ i.e.
\begin{equation}\label{eq:5OPT}
\E\sbr{\cost(\cluster_i, C_{t+1}) \mid C_t, \{D_{t+1} \cap \cluster_i\neq \varnothing\}} \leq 5 \opt_1(P_i).
\end{equation}

We now prove the main lemma.
\begin{lemma}\label{lem:5OPT}
Consider an arbitrary set of centers $C=\{c_1,\dots, c_N\} \subseteq \R^\dimension$ and an arbitrary set $P\subseteq \pointset$. Pick a random point $c$ in $P$ with probability $\Pr(c = x) = \cost(x, C)/\cost(P,C)$. Let $C' = C\cup \{c\}$. Then, $ \E_c\sbr{\cost(P,C')} \leq 5 \opt_1(P)$.
\end{lemma}

\noindent\textbf{Remarks:} Lemma 3.2 in the paper by \citet*{arthur2007k} gives
a bound of $8 \opt_1(P)$. We also show in Appendix~\ref{sec:lb} that our bound is tight (see Lemma~\ref{thm:5-approx-tight}).

\begin{proof}
The cost of any point $y$ after picking
center $c$ equals the squared distance from $y$ to the set of
centers $C' = C \cup \{c\}$, which in turn equals
$\min\{\cost(y,C), \|y-c\|^2\}$. Thus, if a point
$x\in P$ is chosen as a center, then the
cost of point $y$ equals $\min\{\cost(y,C), \norm{x-y}^2\}$.
Since $\Pr(c=x) = \cost(x,C)/\cost(P,C)$, we have
\begin{align*}
\E_c\sbr{\cost(P,C')} = \sum_{\substack{x \in P\\y\in P}} \frac{\cost(x,C)}{\cost(P,C)}\cdot \min\{\cost(y,C), \|x-y\|^2\}. 
\end{align*}
We write the right hand side in a symmetric form
with respect to $x$ and $y$. To this end,
we define function $f$ as follows:
\begin{align*}
f(x,y) = \cost(x,C) \cdot \min\cbr{\norm{x-y}^2, \cost(y,C)} + \cost(y,C) \cdot \min\cbr{\norm{x-y}^2, \cost(x,C)}.
\end{align*}
Note that $f(x,y) = f(y,x)$. Then,
$$\E_c\sbr{\cost(P,C')} =
\frac{1}{2\cost(P,C)} \sum_{(x, y)\in \cluster\times \cluster} f(x,y).
$$

We now give an upper bound on $f(x,y)$ and then use
this bound to finish the proof of Lemma~\ref{lem:5OPT}.

\begin{lemma}\label{lem:fxy}
For any $x,y \in \cluster$, we have $f(x,y) \leq 5 \min\cbr{\cost(x,C),\cost(y,C)} \norm{x-y}^2$.
\end{lemma}
\begin{proof}
Since $f(x,y)$ is a symmetric function with respect to $x$ and $y$, we may assume without loss of generality that $\cost(x,C) \leq \cost(y,C)$. Then, we need to show that $f(x,y) \leq 5 \cost(x,C) \norm{x-y}^2$.
Consider three cases.

\medskip

\noindent\textbf{Case 1:} If $\cost(x,C) \leq \cost(y,C) \leq \norm{x-y}^2$, then
\begin{align*}
    f(x,y) = 2\cost(x,C) \cost(y,C) \leq 2\cost(x,C) \norm{x-y}^2.
\end{align*}

\medskip

\noindent\textbf{Case 2:} If $\cost(x,C) \leq \norm{x-y}^2 \leq \cost(y,C)$, then
\begin{align*}
    f(x,y) = \cost(x,C) \norm{x-y}^2 + \cost(y,C) \cost(x,C).
\end{align*}
By the triangle inequality, we have
$$\cost(y,C) \leq \rbr{\sqrt{\cost(x,C)} + \norm{x-y}}^2 \leq 4 \norm{x-y}^2.$$
Thus, $f(x,y) \leq 5 \cost(x,C) \norm{x-y}^2$.

\medskip

\noindent\textbf{Case 3:} If $\norm{x-y}^2 \leq \cost(x,C) \leq \cost(y,C)$, then
$$f(x,y) = \rbr{\cost(x,C) + \cost(y,C)} \norm{x-y}^2.$$
By the triangle inequality,
$$\cost(y,C) \leq \rbr{\sqrt{\cost(x,C)}+ \norm{x-y}}^2 \leq 4 \cost(x,C).$$
Thus, we have
$f(x,y) \leq 5 \cost(x,C) \norm{x-y}^2$.

In all cases, the desired inequality holds. This concludes the proof of Lemma~\ref{lem:fxy}.
\end{proof}

\medskip

We use Lemma~\ref{lem:fxy} to bound the expected cost of $P$.
Let $\phi^*$ be a vector in $\real^{P}$ with $\phi^*_x = \cost(x,C)$ for any $x \in \cluster$. Then,
$
f(x,y) \leq 5 \min\cbr{\phi^*_x,\phi^*_y} \norm{x-y}^2$.
Since $\cost(P,C) = \sum_{z\in P} \phi^*_z$, we have
$$
\E_c\sbr{\cost(P,C')} \leq
\underbrace{
    \frac{5\sum_{(x, y)\in \cluster\times \cluster} \min\cbr{\phi^*_x,\phi^*_y} \norm{x-y}^2}{2\sum_{z\in P}\phi^*_z}
}_{5F(\phi^*)}.
$$
For arbitrary vector $\phi \in \real_{\geq 0}^{P}$, define the
following function:
\begin{equation}\label{eq:def:F}
F(\phi) = \frac{\sum_{(x,y) \in \cluster\times \cluster}  \min\cbr{\phi_x,\phi_y} \norm{x-y}^2}{2\sum_{z\in P} \phi_z}. 
\end{equation}
We have $\E_c\sbr{\cost(P,C')} \leq 5F(\phi^*)$. Thus,
to finish the proof of Lemma~\ref{lem:5OPT}, it suffices to show that
$F(\phi)\leq \opt_1(P)$ for every $\phi \geq 0$ and particularly for
$\phi=\phi^*$. By Lemma~\ref{lem:5OPT:integral} (which we state and prove below), function
$F(\phi)$ is maximized when $\phi\in \{0,1\}^P$. Let $\phi^{**}$
be a maximizer of $F(\phi)$ in $\{0,1\}^P$
and
$P'=\{x\in P: \phi^{**}_x = 1\}$. Observe that
\begin{align*}
    F(\phi^{**}) =  
\frac{\sum_{(x,y) \in \cluster'\times \cluster'}
\norm{x-y}^2}{2|P'|}
= \opt_1(P').
\end{align*}
Here we used the closed form expression
(\ref{eq:opt-1-closed-form}) for
the optimal cost of cluster $P'$. Since $P'\subset P$, we
have $\opt_1(P')\leq \opt_1(P)$. Thus,
$F(\phi^*)\leq F(\phi^{**})\leq \opt_1(P)$.
\end{proof}

\begin{lemma}\label{lem:5OPT:integral}
There exists a maximizer $\phi^{**}$ of $F(\phi)$ in the region $\{\phi \geq 0\}$
such that $\phi\in \{0,1\}^P$.
\end{lemma}
\begin{proof}
Let $m=|P|$ be the size of the cluster $P$ and
$\Pi$ be the set of all bisections or permutations $\pi: \{1,\dots, m\}\to P$. Partition the set $\{\phi \geq 0\}$ into $m!$ regions (``cones over order polytopes''):
$$\{\phi: \phi \geq 0\}=\cup_{\pi \in \Pi} O_{\pi},$$
where $O_{\pi} = \{\phi : 0 \leq \phi_{\pi(1)} \leq \phi_{\pi(2)}\leq \cdots \leq \phi_{\pi(m)}\}$.
We show that for every $\pi \in \Pi$, there exists a maximizer $\phi^{**}$ of $F(\phi)$ in the
region $O_{\pi}$, such that $\phi^{**}\in \{0,1\}^P$. Therefore, there exists a global
maximizer $\phi^{**}$ that belongs $\{0,1\}^P$

Fix a $\pi\in \Pi$. Denote by $V$ the hyperplane $\{\phi: \sum_{x\in P} \phi_x = 1\}$. Observe that
$F$ is a scale invariant function i.e., $F(\phi)= F(\lambda \phi)$ for every $\lambda > 0$. Thus,
for every $\phi \in O_{\pi}$, there exists
a $\phi' \in O_\pi \cap V$
(namely,  $\phi' = \phi / (\sum_{x\in P} \phi_x)$) such that $F(\phi') = F(\phi)$.
Hence, $\max\{F(\phi): \phi \in O_{\pi}\} = \max\{F(\phi): \phi \in O_{\pi}\cap V\}$.
Note that for $\phi\in V$, the denominator of~(\ref{eq:def:F}) equals 2, and
for $\phi\in O_{\pi}$, the numerator of~(\ref{eq:def:F}) is a linear function of $\phi$.
Therefore, $F(\phi)$ is a linear function in the convex set $O_{\pi}\cap V$. Consequently,
one of the maximizers of $F$ must be an extreme point of $O_{\pi}\cap V$.

The polytope $O_{\pi}\cap V$ is defined by $m$ inequalities and one equality. Thus, for every
extreme point $\phi$ of this polytope, all inequalities $\phi_{\pi(i)}\leq \phi_{\pi(i+1)}$
but one must be tight. In other words, for some $j< m$, we have
\begin{equation}\label{eq:all-but-one-equal}
0 = \phi_{\pi(1)} =  \cdots = \phi_{\pi(j)} < \phi_{\pi(j+1)} = \cdots = \phi_{\pi(m)}. 
\end{equation}
Therefore, there exists a maximizer $\phi$ of $F(\phi)$ in $O_{\pi} \cap V$ satisfying
(\ref{eq:all-but-one-equal}) for some
$j$. After rescaling $\phi$ -- multiplying
all coordinates of $\phi$ by $(m-j)$ -- we obtain a vector $\phi^{**}$ whose first $j$ coordinates
$\phi^{**}_{\pi(1)}, \dots, \phi^{**}_{\pi(j)}$ are zeroes and the last $m-j$ coordinates
$\phi^{**}_{\pi(j+1)}, \dots, \phi^{**}_{\pi(m)}$ are ones. Thus,
$\phi^{**}\in\{0,1\}^P$. Since $F$ is rescaling invariant,
$F(\phi^{**}) = F(\phi)$. This concludes the proof.
\end{proof}

Replacing the bound in Lemma 3.2 from the analysis of \citet{arthur2007k} by our bound from Lemma~\ref{lem:5OPT} gives the following result (see also Lemma~\ref{thm:log-bound}).
\begin{theorem}
The approximation factor of $k$-means++ is at most $5(\ln k + 2)$.
\end{theorem}

We now state an important corollary of
Lemma~\ref{lem:5OPT}.
\begin{corollary}~\label{cor:5opt-martingale}
For every $P\in\calP$, the process $\HH_{t}(P)$ for $k$-means++ is a supermartingale i.e.,
$$\expect{\HH_{t+1}(\pointset) \mid C_t} \leq \HH_{t}(\pointset).$$
\end{corollary}
\begin{proof}
The value of $\HH_{t}(\pointset)$ changes only
if  at step $t$, we cover a yet uncovered cluster $P$. In this case, the value of $\HH_{t+1}(P)$ changes by the new cost of $P$ minus $5\opt(P)$. By Lemma~\ref{lem:5OPT} this quantity is non-positive in expectation. 
\end{proof}

Since the process $\HH_{t}(P)$ is a supermartingale, we have $\E[\HH_{t}(P)] \leq \HH_{0}(P) = 5\opt_1(P)$.
Hence, $\E[H_{t}(P)] \leq \E[\HH_{t}(P)] = 5\opt_1(P)$. Thus,
$\E[H_{t}(X)] \leq 5\opt_k(\pointset)$.
Since $\cost_t(\pointset) = H_t(\pointset) +
U_t(\pointset)$ and we have a bound on the expectation of the covered cost, $H_t(\pointset)$, in the remaining sections, we shall only analyze the uncovered cost $U_t(\pointset)$.

\section{\texorpdfstring{Bi-criteria approximation of $k$-means++}{Bi-criteria approximation of  k-means++}}\label{sec:po-bi-criteria}

In this section, we give a bi-criteria approximation guarantee for $k$-means++.

\begin{theorem}\label{thm:kmeanspp-main}
Let $\cost_{k+\extracenters}\rbr{\total}$ be the cost of the clustering with $k+\extracenters$ centers sampled by the $k$-means++ algorithm. Then,
for $\extracenters \geq 1$, the expected cost
$\expect{\cost_{k+\extracenters}(\pointset)}$ is upper bounded by (below $(a)^+$ denotes $\max(a,0)$).
$$\min\Big\{
2 + \frac{1}{2e} + \Big(\ln{\frac{2k}{\extracenters}}\Big)^+,
1+ \frac{k}{e\rbr{\extracenters-1}}
\Big\}\, 5 \opt_k(\pointset).$$
\end{theorem}

Note that the above approximation guarantee is  the minimum of two bounds: (1)
$2 + \frac{1}{2e} + \ln{\frac{2k}{\extracenters}}$
for $1\leq \Delta \leq 2k$; and (2) $1+ \frac{k}{e\rbr{\extracenters-1}}$
for $\Delta \geq 1$. The second bound is stronger than the first bound when $\extracenters/k \gtrapprox 0.085$.

\subsection{Proof overview of Theorem~\ref{thm:kmeanspp-main}}
We now present a high level overview of the proof and then give a formal proof. Our proof consists of three steps.

First, we prove bound (2) on the expected cost of the clustering returned by $k$-means++ after $k+\Delta$ rounds. We argue that the expected cost of the covered clusters is bounded by $5\opt_k(\pointset)$ (see Section~\ref{sec:framework}) and thus it is sufficient to bound the expected cost of uncovered clusters. Consider an optimal cluster $P\in \calP$. We need to estimate the probability that it is not covered after $k + \Delta$ rounds. We upper bound this probability by the probability that the algorithm does not cover $P$ before it makes $\Delta$ misses (note: after $k+\Delta$ rounds $k$-means++ must make at least  $\Delta$ misses).

In this overview, we make the following simplifying assumptions (which turn out to be satisfied in the worst case for bi-criteria $k$-means++): Suppose that the uncovered cost of cluster $P$ does not decrease before it is covered and equals $U(P)$ and, moreover, the total cost of all covered clusters almost does not change and equals $H(\pointset)$ (this may be the case if one large cluster contributes most of the covered cost, and that cluster is covered at the first step of $k$-means++). Under these assumptions, the probability that $k$-means++ chooses $\Delta$ centers in the already covered clusters and does not choose a single center in $P$ equals $(H(\pointset)/(U(P)+H(\pointset)))^\Delta$. If $k$-means++ does not choose a center in $P$, the \emph{uncovered} cost of cluster $P$ is $U(P)$; otherwise, the \emph{uncovered} cost of cluster $P$ is $0$. Thus, the expected~\emph{uncovered cost} of $P$ is $(H(\pointset)/(U(P)+H(\pointset)))^\Delta U(P)$. It is  easy to show that $(H(\pointset)/(U(P)+H(\pointset)))^\Delta U(P) \leq H(\pointset)/(e (\Delta - 1))$. Thus, the expected \emph{uncovered cost} of all clusters is at most
$$\frac{k}{(e (\Delta - 1))} \E[H(\pointset)]\leq \frac{k}{(e (\Delta - 1))} 5\opt_k(\pointset).$$

Then,  we use ideas from \citet*{arthur2007k},~\citet{dasgupta-notes} to prove the following statement: Let us count the cost of uncovered clusters only when the number of misses after $k$ rounds of $k$-means++ is greater than $\Delta/2$. Then the expected cost of uncovered clusters is at most $O(\log (k/\Delta))\cdot \opt_k(\pointset)$. That is, $\E[H(U_k(\pointset)\cdot \mathbf{1}\{M(C_k)\geq \Delta/2\}]\leq O(\log (k/\Delta))\cdot \opt_k(\pointset)$.

Finally, we combine the previous two steps to get bound (1). We argue that if the number of misses after $k$ rounds of $k$-means++ is less than $\Delta/2$, then
almost all clusters are covered. Hence, we can apply bound (2) to $k'\leq \Delta/2$ uncovered clusters and $\Delta$ remaining rounds of $k$-means++ and get a $5(1+1/(2e))$ approximation. If the number of misses is greater than $\Delta/2$, then the result from the previous step yields an $O(\log (k/\Delta))$ approximation.

\subsection{\texorpdfstring{Analysis of $k$-means++}{Analysis of k-means++}}\label{sec:apx-bi-criteria-kmeans}

In this section, we analyze the bi-criteria $k$-means++ algorithm and prove Theorem~\ref{thm:kmeanspp-main}. To this end, we establish the first and second bounds from Theorem~\ref{thm:kmeanspp-main} on the expected cost of the clustering after $k+\Delta$ rounds of $k$-means. We will start with the second bound.

\subsubsection{\texorpdfstring{Bi-criteria bound for large $\Delta$}{Bi-criteria bound for large Delta}}

\begin{lemma}\label{lem:e_bound}
The following bi-criteria bound holds
$$
    \expect{\cost_{k+\extracenters}\rbr{\total}} \leq 5\rbr{1+ \frac{k}{e\rbr{\extracenters-1}}} \opt_k(\pointset).
$$
\end{lemma}

Consider the discrete time Markov chain $C_t$
associated with $k$-means++ algorithm
(see Section~\ref{sec:framework}). Let $\cluster\in \calP$ be an arbitrary cluster in the optimal solution. Partition all states of the Markov chain into
$k+\Delta$ disjoint groups $\calM_0,\calM_1,\cdots, \calM_{k+\Delta-1}$ and $\calH$. Each set $\calM_i$ contains all states $C$ with $i$
misses that do not cover $P$:
$\calM_i = \cbr{C: M(C) = i, \cluster\cap C = \varnothing}.$
The set $\calH$ contains all states $C$ that cover $P$:
$\calH = \cbr{C: \cluster \cap C \neq \varnothing}$.

We now define a new Markov chain $X_t$. To this end, we first expand the set of states $\{C\}$.
For every state $C$ of the process $C_t$, we create two additional ``virtual'' states $C^a$ and $C^b$. Then, we let $X_{2t} = C_t$ for every even step $2t$,  and
$$
X_{2t+1}=
\begin{cases}
C_t^a, & \mbox{if } C_{t+1}\in \calM_i \\
C_t^b, & \mbox{if } C_{t+1}\in \calM_{i+1} \cup \calH.
\end{cases} $$
for every odd step $2t+1$. We stop $X_t$ when $C_t$ stops or when $C_t$ hits the set $\calH$ (i.e., $C_t\in \calH$).
Loosely speaking,  $X_t$ follows Markov chain $C_t$ but makes additional intermediate stops. When $C_t$ moves from one state in $\calM_i$ to another state in $\calM_i$, $X_{2t+1}$ stops in $C_t^a$; and when  $C_t$ moves from a state in $\calM_i$ to a state in $\calM_{i+1}$ or $\calH$, $X_{2t+1}$ stops in $C_t^b$.

Write transition probabilities for $X_t$:
\begin{align*}
    \prob{X_{2t+1} = C^a \mid X_{2t} =  C} = \frac{U(\pointset,C)- U(\cluster,C)}{\cost(\pointset,C)}, \\
    \prob{X_{2t+1} = C^b \mid  X_{2t} =  C} = \frac{U(\cluster,C) + H(\pointset,C)}{\cost(\pointset,C)},
\end{align*}
and for all $C \in  \calM_{i}$ and $C'= C\cup\{x\}\in \calM_{i}$,
$$
\prob{X_{2t+2} = C' \mid X_{2t+1} = C^a} = \frac{\cost(x,C)}{U(\pointset,C)- U(\cluster,C)},
$$
for all $C \in \calM_i$ and $C'= C\cup\{x\} \in \calM_{i+1} \cup \calH$,
$$
\prob{X_{2t+2} = C' \mid  X_{2t+1} = C^b} = \frac{\cost(x,C)}{U(\cluster,C) + H(\pointset,C)}.
$$
Above, $U(\pointset,C)- U(\cluster,C)$ is the
cost of points in all uncovered clusters except for $P$. If we pick a center from these clusters, we will necessarily cover a new cluster, and therefore $X_{2t+2}$ will stay in $\calM_i$. Similarly,
$U(\cluster,C) + H(\pointset,C)$ is the cost of all covered clusters plus the cost of $P$. If we pick a center from these clusters, then $X_{2t+2}$ will move to $\calM_{i+1}$ or
$\calH$.

Define another Markov chain $\cbr{Y_t}$. The
transition probabilities of $\cbr{Y_t}$
are the same as the transition probabilities
of $X_t$ except $Y$ never visits states in
$\calH$ and therefore for $C\in \calM_i$ and
$C'=C\cup\{x\}\in\calM_{i+1}$, we have
$$
\prob{Y_{2t+2} = C' \mid Y_{2t+1} = C^b} = \frac{\cost(x,C)}{H(\pointset,C)}. 
$$

We now prove a lemma that relates probabilities of visiting states by $X_t$ and $Y_t$.
\begin{lemma}~\label{lem:coupling}
For every $t\leq k+\Delta$ and states $C' \in \calM_i$, $C''\in \calM_\extracenters$, we have
$$
\frac{\prob{C'' \in \cbr{X_j} \mid X_{2t} = C'}}{\prob{C'' \in \cbr{Y_j} \mid Y_{2t} = C'}} \leq \rbr{\frac{\HH(\pointset,C'')}{\HH(\pointset,C'')+U(\cluster,C'')}}^{\Delta - i}
$$
where $\{C'' \in \cbr{X_j}\}$ and $\{C'' \in \cbr{Y_j}\}$ denote the events $X$ visits
$C''$ and $Y$ visits $C''$, respectively.
\end{lemma}

\begin{proof}
Consider the unique path $p$ from $C'$ to $C''$
in the state space of $X$ (note that the transition graphs for $X$ and $Y$ are directed trees).
The probability of transitioning from $C'$ to
$C''$ for $X$ and $Y$ equals the product of
respective transition probabilities for every
edge on the path. Recall that transitions probabilities for $X$ and $Y$ are the same for all states but $C^b$, where $C\in \cup_j\calM_j$.
The number of such states on the path $p$ is
equal to the number transitions from $\calM_j$
to $\calM_{j+1}$, since $X$ and $Y$ can get
from $\calM_j$ to $\calM_{j+1}$ only through
a state $C^b$ on the boundary of
$\calM_j$ and $\calM_{j+1}$. The number of
transitions from $\calM_j$ to $\calM_{j+1}$
equals $\Delta - i$. For each state $C^b$ on the
path, the ratio of transition probabilities
from $C^b$ to the next state $C \cup \{x\}$ for Markov chains
$X$ and $Y$ equals
$$
 \frac{H(\pointset,C)}{U(\cluster,C) + H(\pointset,C)}\leq
 \frac{\HH(\pointset,C'')}{U(\cluster,C'') + \HH(\pointset,C'')},
$$
here we used that (a) $U(P,C)\geq U(P,C'')$ since
$U_t(P)$ is a non-increasing process; and (b)
$H(P,C)\leq \HH(P,C'')$ since
$H_t(P)\leq \HH_{t'}(P)$
if $t \leq t'$ (see Section~\ref{sec:framework}).
\end{proof}
We now prove an analog of Corollary~\ref{cor:5opt-martingale} for
$\HH(\pointset,Y_j)$.
\begin{lemma}~\label{lem:supermartingale}
$\HH(\pointset,Y_t)$ is a supermartingale.
\end{lemma}
\begin{proof}
If $Y_j = C$, then $Y_{j+1}$ can only be in $\cbr{C^a,C^b}$. Since $\HH(\pointset,C^a) = \HH(\pointset,C^b) = \HH(\pointset,C)$, we have $\expect{\HH(\pointset,Y_{j+1}) \mid  Y_j = C} = \HH(\pointset,Y_j)$.

If $Y_j = C^a$, then $Y_{j+1} = C'$ where the new center $c$ should be in uncovered clusters with respect to $C_t$.
$$
    \expect{H(\cluster',Y_{j+1}) \mid Y_j = C^a, c\in \cluster'} \leq 5\opt_1(\cluster'),
$$
which implies
$$
    \expect{\HH(\cluster',Y_{j+1}) \mid Y_j = C^a, c \in \cluster'} \leq \HH(\cluster',Y_j).
$$
Therefore, we have
$$
    \expect{\HH(\pointset,Y_{j+1}) \mid Y_j= C^a} \leq \HH(\pointset,Y_j).
$$
If $Y_j = C^b$, then for any possible state $C'$ of $Y_{j+1}$,  the new center should be in covered clusters with respect to $C$. By definition, we must have
$\HH(\pointset,C') = \HH(\pointset,C) = \HH(\pointset,C^b)$. Thus, it holds that $\expect{\HH(\pointset,Y_{j+1}) \mid Y_j = C^b} = \HH(\pointset,Y_j)$.

Combining all these cases, we get $\cbr{\HH(\pointset,Y_j)}$ is a supermartingale.
\end{proof}

We now use Lemma~\ref{lem:coupling} and Lemma~\ref{lem:supermartingale} to bound the expected uncovered cost of $P$ after
$k+\extracenters$ rounds of $k$-means++.

\begin{lemma}~\label{lem:e_bound-one-cluster}
For any cluster $P\in \calP$ and $t\leq k+\Delta$, we have
$$
\expect{U_{k+\extracenters}(\cluster) \mid C_t} \leq \frac{\HH_t(\pointset)}{e(\extracenters-M(C_t)-1)}.
$$
\end{lemma}
\begin{proof}
Since $k$-means++ samples $k+\extracenters$ centers and the total number of clusters in the optimal solution $\calP$ is $k$, $k$-means++ must make $\Delta$ misses. Hence, the process $\cbr{X_t}$ which follows $k$-means++
must either visit a state in $\calM_{\geq \extracenters}$
or stop in $\calH$ (recall that we stop process $X_t$ if it reaches $\calH$).

If $\cbr{X_t}$ stops in group $\calH$, then the cluster $P$ is covered which means that $U_{k+\extracenters}(\cluster) = 0$.
Let $\partial \calM_{\Delta}$ be the frontier of
$\calM_{\Delta}$ i.e., the states that $X_t$
visits first when it reaches $\calM_\Delta$
(recall that the transition graph of $X_t$ is
a tree). The expected cost
$\expect{U_{k+\extracenters}(\cluster) \mid C_t}$ is upper bounded by the expected uncovered cost of $P$ at time when $C_t$ reaches $\MM_{\Delta}$.
Thus,
$$
    \expect{U_{k+\extracenters}(\cluster) \mid C_t} \leq \sum_{C \in \MM_\extracenters} \prob{C \in \cbr{X_j} \mid C_t} U(\cluster,C).
$$
Observe that by Lemma~\ref{lem:coupling},  for any $C \in \MM_{\Delta}$, we have
\begin{align*}
\prob{C \in \cbr{X_j} \mid C_t} U(P,C)
\leq \prob{C \in \cbr{Y_j} \mid C_t} \rbr{\frac{\HH(\pointset,C)}{\HH(\pointset,C)+U(\cluster,C)}}^{\extracenters'} U(P,C).
\end{align*}

Let $f(x) = x(1/(1+x))^{\extracenters'}$. Then, $f(x)$ is maximized at $x = 1/(\extracenters'-1)$ and the maximum value $f(1/(\extracenters'-1)) = 1/(e(\extracenters'-1))$. Therefore, for every $C \in \MM_\extracenters$, we have
\begin{align*}
    \pr[C \in \cbr{X_j} \mid C_t] U(P,C) &\leq \prob{C \in \cbr{Y_j} \mid C_t} f\rbr{\frac{U(\cluster,C)}{\HH(\pointset,C)}} \HH(\pointset,C) \\
    &\leq \prob{C \in \cbr{Y_j} \mid C_t} \frac{\HH(\pointset,C)}{e(\extracenters'-1)}.
\end{align*}
Let $\tau = \min\cbr{j: Y_j \in \MM_\extracenters}$ be the stopping time when $Y_j$ first visits $\MM_\extracenters$. We get
\begin{align*}
    \sum_{C \in \MM_\extracenters} \prob{C \in \cbr{Y_j} \mid C_t} \HH(\pointset,C) = \expect{\HH(\pointset,Y_\tau) \mid C_t}.
\end{align*}
By Lemma~\ref{lem:supermartingale},  $\HH(\pointset,Y_j)$ is a supermartingale. Thus,
by the optional stopping theorem,
$$
    \expect{\HH(\pointset,Y_\tau) \mid C_t} \leq \HH(\pointset,C_t).
$$
Therefore, we have
\begin{align*}
    \expect{U_{k+\extracenters}(\cluster) \mid C_t} \leq \frac{\HH_t(\pointset)}{e(\extracenters-M(C_t)-1)},
\end{align*}
This concludes the proof.
\end{proof}
We now add up bounds from  Lemma~\ref{lem:e_bound-one-cluster} with
$t= 0$
for all clusters $P\in \calP$ and obtain
Lemma~\ref{lem:e_bound}.

\subsection{\texorpdfstring{Bi-criteria bound for small $\Delta$}{Bi-criteria bound for small Delta}}
In this section, we give another bi-criteria
approximation guarantee for $k$-means++.

\begin{lemma}\label{thm:log-bound}
Let $\cost_{k+\extracenters}(\pointset)$ be the cost of the the clustering resulting from sampling $k+\Delta$ centers according to the $k$-means++ algorithm (for $\Delta \in \{1,\dots, 2k\}$). Then,
$$
    \expect{\cost_{k+\extracenters}(X)} \leq 5
    \rbr{2 + \frac{1}{2e} + \ln{\frac{2k}{\extracenters}}}
    \opt_{k}(X).
$$
\end{lemma}
\begin{proof}
Consider  $k$-means++ clustering algorithm and the corresponding random process $C_t$.
Fix a $\kappa\in\{1,\dots,k\}$.
Let $\tau$ be the first iteration\footnote{Recall, that $K(C_t)$ is a non-increasing stochastic
process with $K(C_0) = k$.} (stopping time) when $K(C_\tau) \leq \kappa$ if $K(C_k) \leq \kappa$;
and $\tau = k$, otherwise.
We refer the reader to Section~\ref{sec:framework} for
definitions of $M(C_t)$, $U_t(X)=U(X,C_t)$,
and $K(C_t)$.

We separately analyze the cost of uncovered clusters after the first $\tau$ steps and the last $k'-\tau$ steps, where $k'=k+\Delta$ is the total number of centers chosen by $k$-means++.

The first step of our proof follows the analysis of $k$-means++ by~\citet{dasgupta-notes}, and by \citet*{arthur2007k}. Define a potential function $\Psi$ (see~\citealt{dasgupta-notes}):
$$
    \potential_t \coloneqq \cfrac{M(C_t) U(X, C_t)}{K(C_t)}.
$$
If $K(C_t)= 0$, then
$M(C_t)$ and $U(X,C_t)$ must be $0$
and we let $\Psi_t =0$

We use the following result by~\citet{dasgupta-notes} to estimate
$\E[\Psi_{\tau}(X)]$ in Lemma~\ref{lem:many-misses}.
\begin{lemma}[\citet{dasgupta-notes}]\label{lem:potential}
For any $0\leq t \leq k$, we have
$$
\expect{\potential_{t+1} - \potential_{t}\mid C_t}\leq
\frac{H(X, C_t)}{K(C_t)}.
$$
\end{lemma}

\begin{lemma}\label{lem:many-misses}  Then,
the following bound holds:
$$\E[\Psi_{\tau}(X)]
\leq 5\rbr{1+\ln \rbr{\frac{k}{\kappa+1}} } \opt_k(X).$$
\end{lemma}
\begin{proof}
Note that $\potential_{1} = 0$ as
$M(C_1)= 0$. Thus,
$$
\E[\potential_{\tau}]
\leq \sum_{t=1}^{\tau-1}
\E \big[\potential_{t+1} -  \potential_{t}\big]\leq
\E\Big[\sum_{t=1}^{\tau-1}\frac{H(X, C_t)}{K(C_t)}\Big].
$$
Using the inequality $H(X, C_t)\leq
\HH_k(X)$ (see Section~\ref{sec:framework}), we get:
$$
\E[\potential_{\tau}]
\leq
\E\Big[\sum_{t=1}^{\tau-1}\frac{\HH_k(X)}{K(C_t)}\Big]
\leq \E\Big[\HH_k(X)\cdot
\sum_{t=1}^{\tau-1}\frac{1}{K(C_t)}\Big].
$$
Observe that $K(C_1),\dots, K(C_{\tau-1})$ is
a non-increasing sequence in which two consecutive
terms are either equal or $K(C_{i+1}) = K(C_i) - 1$.
Moreover, $K(C_1) =k$ and $K(C_{\tau - 1}) > \kappa$. Therefore,
by Lemma~\ref{lem:harmonic:series} (see below),
for every realization $C_0,C_1,\dots, C_{\tau}$,
we have:
$$\sum_{t=1}^{\tau -1}\frac{1}{K(C_t)}\leq 1 + \log \nicefrac{k}{(\kappa + 1)}.$$
Thus,
$$
\E[\potential_{\tau}] \leq
(1 + \log \nicefrac{k}{(\kappa + 1)}) \E[\HH_k(X)]
\leq
5(1 + \log \nicefrac{k}{(\kappa + 1)})\;\opt_k(X). $$
This concludes the proof.
\end{proof}

Let $\kappa = \floor{(\Delta - 1)/2}$.
By Lemma~\ref{lem:many-misses}, we have
$$\E\Big[\cfrac{M(C_{\tau}) U_{\tau}(X)}{K(C_{\tau})}\Big] \leq 5\rbr{1+\ln \frac{2k}{\Delta}} \opt_k(X).$$
Since $U_{t}(X)$ is a non-increasing
stochastic process, we have $\E[U_{k+\Delta}(X)]\leq \E[U_{\tau}(X)]$. Thus,
$$\E\Big[\cfrac{M(C_{\tau})}{K(C_{\tau})}\cdot
U_{k+\Delta}(X)\Big] \leq 5\rbr{1+\ln \frac{2k}{\Delta}} \opt_k(X).$$
Our goal is to bound $\E[U_{k'}(X)]$.
Write,
\begin{align*}
\E[U_{k'}(X)] =
\E\Big[\cfrac{M(C_{\tau})}{K(C_{\tau})}\cdot
U_{k'}(X)\Big]+
\E\Big[\cfrac{K(C_\tau) - M(C_{\tau})}{K(C_{\tau})}\cdot
U_{k'}(X)\Big].
\end{align*}
The first term on the right hand side is upper bounded by $5\big(1+\ln \frac{2k}{\Delta}\big) \opt_k(X)$. We now estimate the second term,
which we denote by $(*)$.

Note that
$K(C_t) - M(C_{t}) = k - t$,
since the number of uncovered clusters after
$t$ steps of $k$-means++ equals
the number of misses plus the number of steps remaining. Particularly, if $\tau = k$,
we have $K(C_\tau) - M(C_{\tau}) =
K(C_k) - M(C_k) = 0$. Consequently, if $\tau = k$, then the second term $(*)$ equals 0.
Thus, we only need to consider the case, when $\tau < k$. Note that in this case $K(C_{\tau}) = \kappa$. By
Lemma~\ref{lem:e_bound} (applied to all uncovered
clusters), we have
$$
\E[U_{k'}(X) \mid C_{\tau},\tau]
\leq \frac{K(C_{\tau})}{e(\Delta' - 1)}\HH_{\tau}(X),
$$
where $\Delta' = \Delta - M(C_{\tau})$.

Thus,
\begin{align*}
\E\Big[\cfrac{K(C_\tau) - M(C_{\tau})}{K(C_{\tau})}\cdot
U_{k'}(X) \mid C_{\tau},\tau \Big]
\leq
\cfrac{K(C_\tau) - M(C_{\tau})}{K(C_{\tau})}\cdot
\frac{K(C_{\tau})}{e(\Delta' - 1)}
\cdot \HH_{\tau}(X) = (**).
\end{align*}
Plugging in $K(C_{\tau})=\kappa$ and
the expression for $\Delta'$ (see above), and using that $\kappa\leq (\Delta - 1)/2$,
we get
$$
(**)=\cfrac{\kappa - M(C_{\tau})}{e
(\Delta - M(C_{\tau}) - 1)}.
\cdot \HH_{\tau}(X)\leq \frac{1}{2e}\HH_{\tau}(X).
$$
Finally, taking the expectation over all $C_{\tau}$, we obtain the bound
$$
\E\Big[\cfrac{K(C_\tau) - M(C_{\tau})}{K(C_{\tau})}\cdot
U_{k'}(X)\Big]
\leq \frac{5\opt_1(X)}{2e}.
$$
Thus, $\E[U_{k'}(X)] \leq 5(1 + \nicefrac{1}{2e} + \ln \nicefrac{2k}{\Delta})\opt_k(X)$.
Therefore,
$$
\E[\cost_{k'}(X)] = \E[H_{k'}(X)] + U_{k'}(X) \leq 5 \big(2 + \frac{1}{2e} + \ln \frac{2k}{\Delta}\big)\;\opt_k(X).
$$
\end{proof}

We now prove Lemma~\ref{lem:harmonic:series}.
\begin{lemma}~\label{lem:harmonic:series}
For any $t \leq k$ integers $a_1 \geq a_2 \geq \cdots \geq a_t$ such that $a_1 = k$, $a_t > \kappa$ and $a_i -a_{i+1} \in \cbr{0,1}$ for all $1\leq i < t$, the following inequality holds
$$
    \sum_{i=1}^t \frac{1}{a_i} \leq 1+ \log\rbr{\frac{k}{\kappa+1}}.
$$
\end{lemma}

\begin{proof}
It is easy to see that the sum is maximized when $t=k$, and the sequence
$a_1,\dots,a_k$ is as follows:
$$
\underbrace{\frac{1}{k},\frac{1}{k-1},
\dots, \frac{1}{\kappa+2}}_{(k-(\kappa+1)) \text{ terms}},
\underbrace{\frac{1}{\kappa+1},\dots, \frac{1}{\kappa+1}}_{(\kappa + 1) \text{ terms}}.
$$
The sum of the first $(k-(\kappa+1))$ terms is upper bounded by
$$\int_{1/(\kappa + 1)}^{1/k}\frac{1}{x}\;dx = \ln \frac{k}{\kappa + 1}.$$
The sum of the last $(\kappa + 1)$
terms is 1.
\end{proof}

\section{\texorpdfstring{Analysis of $k$-means$\parallel$}{Analysis of k-means Parallel}}\label{sec:po-kmeans-parallel}

In this section, we give a sketch of analysis for the $k$-means$\parallel$ algorithm.
Specifically, we show upper bounds on the expected cost of the solution after $T$ rounds.

\begin{theorem}~\label{thm:kpp_main}
The expected cost of the clustering returned by $k$-means$\parallel$ algorithm after $T$ rounds are upper bounded as follows:
\begin{align*}
    \text{for $\ell < k$, }\qquad \expect{\cost_{T+1}(\total)} &\leq \rbr{e^{-\frac{\ell}{k}}}^T \expect{\cost_1(\total)} + \frac{5\opt_k(\pointset)}{1-e^{-\frac{\ell}{k}}};
    \\\text{for $\ell \geq k$, }\qquad \expect{\cost_{T+1}(\total)} &\leq \rbr{\frac{k}{e\ell}}^T \expect{\cost_1(\total)} + \frac{5\opt_k(\pointset)}{1-\nicefrac{k}{e\ell}}.
\end{align*}
\end{theorem}

\noindent\textbf{Remark:} For the second bound ($\ell \geq k$), the additive term $5\opt_k(\pointset)/(1-k/(e\ell)) \leq 8 \opt_k(\pointset)$.

The probability that a point is sampled by $k$-means$\parallel$ is strictly greater than the probability that it is sampled by $k$-means$\kpp$ since $1-e^{-\lambda} < \lambda$ for all $\lambda > 0$. Thus, for every round, we can couple $k$-means$\kpp$ and $k$-means$\parallel$ so that each point sampled by $k$-means$\kpp$ is also sampled by $k$-means$\parallel$. Thus, the expected cost returned by $k$-means$\parallel$ is at most the expected cost returned by $k$-means$\kpp$. In the following analysis, we show an upper bound for the expected cost of the solution returned by $k$-means$\kpp$.

As a thought experiment, consider a modified $k$-means$\kpp$ algorithm. This algorithm is given the set $\pointset$, parameter $k$, and additionally
the optimal solution $\calP=\{P_1,\dots, P_k\}$. Although this modified algorithm is useless in practice as we do not know the optimal solution in advance, it will be helpful for our analysis.

In every round $t$, the modified algorithm first draws independent Poisson random variables $Z_t(P_i)\sim\Pois(\lambda_t(P_i))$ for every cluster $i \in \{1,\dots,k\}$ with rate $\lambda_t(P_i) = \sum_{x\in P_i} \lambda_t(x)$. Then, for each $i\in\{1,\dots, k\}$, it samples $Z_t(P_i)$ points $x\in P_i$ with repetitions from $P_i$, picking every point $x$ with probability $\lambda_t(x)/\lambda_t(P_i)$ and adds them to the set of centers $C_t$. We assume that points in every set $C_t$ are ordered in the same way as they were chosen by this algorithm.

We claim that the distribution of the output sets $C_T$ of this algorithm is exactly the same as in the original $k$-means$\kpp$ algorithm. Therefore, we can analyze the modified algorithm instead of $k$-means$\kpp$, using the framework described in
Sections~\ref{sec:framework}.

\begin{lemma}~\label{lem:kmpp}
The sets $C_t$ in the original and modified $k$-means$\kpp$ algorithms are identically distributed.
\end{lemma}
\begin{proof}  Consider $|P_i|$ independent Poisson point processes $N_x(a)$ with rates $\lambda_t(x)$, where $x\in P_i$ (here, we use variable $a$ for time). Suppose we add a center $x$ at step $t$ of the algorithm if $N_x(t)\geq 1$. On the one hand, the probability that we choose $x$ is equal to $1-e^{-\lambda_t(x)}$ which is exactly the probability that $k$-means\kpp\; picks $x$ as a center at step $t$. On the other hand, the sum  $N_{P_i}= \sum_{x\in P_i} N_x$ is a Poisson point process with rate $\lambda_t(P_i)$. Thus, the total number of jumps in the interval $[0,1]$ of processes $N_x$ with $x\in P_i$ is distributed as $Z_t(P_i)$. Moreover, the probability that $N_x$ jumps at time $a$ conditioned on the event that $N_{P_i}$ jumps at time $a$ is $\lambda_t(x)/\lambda_t(P_i)$. Thus, for every jump of $N_{P_i}$, we choose one random center $x$ with probability $\lambda_t(x)/\lambda_t(P_i)$.
\end{proof}

\begin{lemma}~\label{lem:parallel_1}
For $k$-means$\parallel$  algorithm with parameter $\ell$, the
following  bounds hold:
\begin{align*}
    \text{for $\ell < k$, }\qquad  \expect{\cost_{t+1}(\total) } &\leq e^{-\frac{\ell}{k}} \cdot \expect{\cost_t(\total)} + 5\opt_k(\pointset);
    \\\text{for $\ell \geq k$, }\qquad \expect{\cost_{t+1}(\total) } &\leq \rbr{\frac{k}{e\ell}} \cdot \expect{\cost_t(\total)} + 5\opt_k(\pointset).
\end{align*}
\end{lemma}

\begin{proof}

Since the expected cost returned by $k$-means$\parallel$ is at most the expected cost returned by $k$-means$\kpp$, we analyze the expected cost of the clustering after one step of $k$-means\kpp.

If the algorithm covers cluster $P_i$ at round $t$, then at the next round, its uncovered cost equals $0$. The number of centers chosen in $P_i$ is determined by the Poisson random variable $Z_{t+1}(P_i)$. Hence, $P_i$ is uncovered at round $t+1$ only if $Z_{t+1}(P_i)=0$. Since $U_t(\cluster_i)$ is non-increasing in $t$ and $U_t(\cluster_i) \leq \cost_t(\cluster_i)$, we have
\begin{align*}
    \expect{U_{t+1}(P_i) \mid \centerset_t} \leq  \prob{Z_{t+1}(\cluster_i) = 0} U_t(\cluster_i) \leq \exp\rbr{- \frac{\ell\, \cost_t(\cluster_i)}{\cost_t(\total)}} \cost_t(\cluster_i).
\end{align*}
Define two function: $f(x) = e^{-x}\cdot x$; and $g(x) = f(x)$ for $x \in [0,1]$ and
$g(x) =e^{-1}$ for $x \in [1,\infty)$. Then,
\[ \expect{U_{t+1}(\pointset) \mid \centerset_t} \leq \rbr{\frac{1}{k}\sum_{i=1}^k f\rbr{\frac{\ell \cost_t(\cluster_i)}{\cost_t(\pointset)}}} \frac{k\cost_t(\pointset)}{\ell}. \]
Since $g(x)\leq f(x)$, and $g(x)$ is concave for $x\geq 0$, we have
\begin{align*}
    \expect{U_{t+1}(\pointset) \mid \centerset_t}
    \leq \rbr{\frac{1}{k} \sum_{i=1}^k g\rbr{\frac{\ell\; \cost_t(\cluster_i)}{\cost_t(\pointset)}}} \frac{k\cost_t(\pointset)}{\ell} \leq g\rbr{\frac{\ell}{k}} \frac{k \cost_t(\pointset)}{\ell}.
\end{align*}
Here, we use that $\sum_i \cost_t(P_i) = \cost_t(\pointset)$.

Therefore, for $ \ell \leq k$, we have
$$
    \expect{U_{t+1}(\pointset) \mid \centerset_t} \leq \rbr{e^{-\frac{\ell}{k}}} \cost_t(\pointset);
$$
and for $\ell \geq k$, we have
$$
    \expect{U_{t+1}(\pointset) \mid \centerset_t} \leq \rbr{\frac{k}{e\ell}} \cost_t(\pointset).
$$

Similar to Corollary~\ref{cor:5opt-martingale}, the process $\HH_t(P)$ for $k$-means$\kpp$ is also a supermartingale, which implies $\expect{H_{t+1}(\pointset) } \leq 5\opt_k(\pointset)$.
This concludes the proof.
\end{proof}

\begin{proof}[Proof of Theorem~\ref{thm:kpp_main}]
Applying the bound from Lemma~\ref{lem:parallel_1} for $t$ times, we get the following results. For $\ell \leq k$,
$$
    \expect{\cost_{t+1}(\total)} \leq \rbr{e^{-\frac{\ell}{k}}}^t \expect{\cost_1(\total)} + 5 \opt_k(\pointset) \eta_t,
$$
where
$
    \eta_t = \sum_{j=1}^{t} \rbr{e^{-\frac{\ell}{k}}}^{j-1} < \frac{1}{1-e^{-\frac{\ell}{k}}}.
$ For $\ell \geq k$,
$$
    \expect{\cost_{t+1}(\total)} \leq \rbr{\frac{k}{e\ell}}^t \expect{\cost_1(\total)} + 5 \opt_k(\pointset) \eta_t,
$$
where
$
    \eta_t = \sum_{j=1}^{t} \rbr{\frac{k}{e\ell}}^{j-1} \leq \frac{1}{1-\frac{k}{e\ell}}.
$
\end{proof}

\begin{corollary}~\label{cor:kpp_9opt}
Consider a data set $\pointset$ with more than $k$ distinct points. Let
$$T = \ln\E\bigg[\frac{\cost_1(\total)}
{\opt_k(\pointset))}\bigg]$$
and $\ell > k$. Then, after $T$ rounds of $k$-means$\parallel$, the expected cost of clustering $\expect{\cost_T(\total)}$ is at most $9\opt_k(\pointset)$.
\end{corollary}
\section{\texorpdfstring{Exponential Race $k$-means++ and Reservoir Sampling}{Exponential Race k-means++ and Reservoir Sampling}}\label{sec:pois-kmeans-pp}

In this section, we show how to implement $k$-means++
algorithm in parallel using $R$ passes over the data set. This implementation, which we refer to as $k$-means$\kpois$ (exponential race $k$-means++), is very similar to $k$-means$\parallel$, but has stronger theoretical guarantees. Like $k$-means$\parallel$, in every round,
$k$-means$\kpois$ tentatively selects $\ell$ centers, in expectation. However, in the same round, it removes some of the just selected centers (without making another pass over the data set). Consequently, by the end of each iteration, the algorithm keeps at most $k$ centers.

We can run $k$-means$\kpois$ till it samples exactly
$k$ centers; in which case, the distribution of
$k$ sampled centers is identical to the distribution of the regular $k$-means++, and the expected number of rounds or
passes over the data set $R$ is upper bounded by
$$O\bigg(\frac{k}{\ell}+ \log\frac{\OPT_1(\pointset)}{\OPT_k(\pointset)}\bigg).$$
We note that $R$ is never greater than $k$. We can also run this algorithm for at most $R^*$ rounds. Then, the expected cost of the clustering is at most
$$
5 (\ln k + 2) \OPT_k(\pointset) +
5R^*\bigg(\frac{4k}{e\ell R^*}\bigg)^{R^*}\cdot \OPT_1(\pointset).
$$

\subsection{Algorithm}
In this section, we give a high level description of our $k$-means$\kpois$ algorithm. In Section~\ref{sec:lazy}, we show how to efficiently implement $k$-means$\kpois$ using lazy updates and explain why our algorithm makes $R$ passes over the data set.

The algorithm simulates $n$ continuous-time stochastic processes. Each stochastic process is associated with one of the points in the data set. We denote the process corresponding to $x\in \pointset$ by $P_t(x)$. Stochastic process $P_t(x)$ is a Poisson process with variable arrival rate $\lambda_t(x)$.

The algorithm  chooses the first center $c_1$ uniformly at random in $\pointset$ and sets the arrival rate of each process $P_t(x)$ to be $\lambda_t(x)=\cost(x,\{c_1\})$. Then, it waits till one of the Poisson processes $P_t(x)$ jumps. When process $P_t(x)$ jumps, the algorithm adds the point $x\in\pointset$ (corresponding to that process) to the set of centers $C_t$ and updates the arrival rates of all processes to be
$$\lambda_t(y)=\cost(y,C_t)$$
for all $y\in \pointset$. Note that if $y$ is a center, then the arrival rate  $\lambda_t(y)$ is $0$.

The algorithm also maintains a round counter $R$. In the lazy version of this algorithm (which we describe in the next section), the algorithm makes a pass over the data set and samples a new batch of centers every time this counter is incremented. Additionally, at the end of each round, the algorithm checks if it chose at least one center in that round, and in the unlikely event that it did not, it selects one center with probability proportional to the costs of the points.

Initially, the algorithm sets $R=0$, $t_0 = 0$, and $t_1 =\ell/\cost(\pointset,\{c_1\})$. Then, at each time point $t_i$ ($i \geq 1$), we increment $R$ and compute
$$t_{i+1}= t_i + \ell / \cost(\pointset,C_{t_i}),$$
where $C_{t_i}$ is the set of all centers selected before time $t_i$. We refer to the time frame
$[t_{i-1},t_{i}]$ for $i\geq 1$ as the $i$-th round. The algorithm stops when one of the following conditions holds true (1)~the number of sampled centers is $k$; or (2) the round counter $R$ equals the prespecified threshold
$R^*$, which may be finite or infinite.

Before analyzing this algorithm, we mention that every Poisson process $P_t$ with a variable arrival rate $\lambda_t$ can be coupled with a Poisson process $Q_s$ with rate $1$. To this end, we substitute the variable
$$s(t) = \int_0^t \lambda_{\tau}d\tau,$$
and let
$$P_t \equiv Q_{s(t)}.$$
Observe that the expected number of arrivals for process $Q_s$ in the infinitesimal interval
$[s,s+ds]$ is $ds = \lambda_{t}dt$ which is exactly the same as for process $P_t$.

It is convenient to think about the variables $s$ as ``current position'', $t$ as ``current time'', and $\lambda_{t}$ as ``current speed'' of $s$. To generate process $P_t(x)$, we can first generate Poisson process $Q_s(x)$ with arrival rate $1$ and then move the position $s_t(x)$ with speed $\lambda_t(x)$. The process $P_t(x) = Q_{s_t(x)}(x)$ is a Poisson process with variable arrival rate $\lambda_t(x)$.

\medskip

\begin{theorem}\label{thm:main-poisson-kpp}
I. If the number of rounds is not bounded (i.e., $R^*=\infty$), then the distribution of centers returned by $k$-means$\kpois$ is identical to the distribution of centers returned by $k$-means++.

II. Moreover, the expected number of rounds $R$ is upper bounded by
$$(1+o_k(1))\cdot\bigg(\ceil{\frac{k}{\ell}} + \log\frac{2\OPT_1(\pointset)}{\OPT_k(\pointset)} \bigg),$$
and never exceeds $k$.

III. If the threshold $R^*$ is given ($R^*<\infty$), then the cost of the solution after $R^*$ rounds is upper bounded by
$$
5 (\ln k + 2) \OPT_k(\pointset) + 2 R^{*}
\bigg(\frac{4k}{e\ell R^*}\bigg)^{R^*}\cdot \OPT_1(\pointset).
$$
\end{theorem}
\begin{proof}[Proof of Part I]
For the sake of analysis, we assume that after the algorithm outputs solution $C$, it does not terminate, but instead continues to simulate Poisson processes $P_t(x)$.
It also continues to update the set $C_t$ (but, of course, not the solution) and the arrival rates $\lambda_t(x)$ till the set $C_t$ contains $k$ centers. Once $|C_t|=k$, the algorithm stops updating the set of centers $C_t$ and arrival rates but still simulates
continuous-time processes $P_t(x)$. Clearly, this additional phase of the algorithm does not affect the solution since it starts after the solution is already returned to the user.

We prove by induction on $i$ that the first $i$ centers $c_1,\dots,c_i$ have exactly the same joint distribution as in $k$-means++. Indeed, the first center $c_1$ is drawn uniformly at random from the data set $\pointset$ as in $k$-means++. Suppose centers $c_1,\dots,c_i$ are already selected. Then, we choose the next center $c_{i+1}$ at the time of the next jump of one of the Poisson processes $P_t(x)$.
Observe that the conditional probability that a particular process $P_t(x)$ jumps given that one of the processes $P_t(y)$ ($y\in \pointset$) jumps is proportional to $\lambda_t(x)$, which in turn equals the current $\cost(x,C_t)$ of point $x$. Hence, the distribution of center $c_{i+1}$ is the same as in $k$-means++. This completes the proof of item I.
\end{proof}
\begin{proof}[Proof of Part II]
We now show items II and III. Define process
$$P_t(\pointset)=\sum_{x\in \pointset} P_t(x).$$
Its rate $\lambda_t(\pointset)$ equals $\sum_{x\in \pointset} \lambda_t(x)$. We couple this process
with a Poisson $Q_s(\pointset)$ with arrival rate $1$ as discussed above. We want to estimate the number of centers chosen by the algorithm in the first $R'$ rounds. To this end, we count the number of jumps of the Poisson process $P_t(\pointset)$ (recall that we add a new center to $C_t$ whenever $P_t(\pointset)$ jumps unless $|C_t|$ already contains $k$ centers). The number of jumps equals $P_{t_{R'}}$ which, in turn, equals $Q_{s_{R'}}$ where $s_{R'}(\pointset)$ is the position of $s(\pointset)$
at time $t_{R'}$:
$$
s_{R'}(\pointset) = \int_{0}^{t_{R'}} \lambda_{\tau}(\pointset)\; d\tau =
\sum_{i=0}^{R'-1}\int_{t_i}^{t_{i+1}}
\lambda_{\tau}(\pointset) \;d\tau \geq
\sum_{i=0}^{R'-1} (t_{i+1}-t_i)\cdot \lambda_{t_{i+1}}(\pointset).
$$
Here, we used that $\lambda_{t}(\pointset)$ is non-increasing, and thus,
$\lambda_{t_{i+1}}(\pointset)\leq \lambda_{\tau}(\pointset)$ for all $\tau \in [t_i,t_{i+1}]$. We now recall that $(t_{i+1}-t_i)= \ell / \cost(\pointset, C_{t_i})$ and
$\lambda_{t_{i+1}}(\pointset) = \cost(\pointset, C_{t_{i+1}})$. Hence,
$$
s_{R'}(\pointset) \geq \ell \sum_{i=0}^{R'-1} \frac{\cost(\pointset, C_{t_{i+1}})}{\cost(\pointset, C_{t_{i}})}.
$$
By the inequality of arithmetic and geometric means, we have
\begin{align}
\label{eq:sR}
s_{R'}(\pointset) &\geq \ell \cdot R' \Bigg(\prod_{i=0}^{R'-1} \frac{\cost(\pointset, C_{t_{i+1}})}{\cost(\pointset, C_{t_{i}})}\Bigg)^{\nicefrac{1}{R'}} =
\ell \cdot R' \Bigg(\frac{\cost(\pointset, C_{t_{R'}})}{\cost(\pointset, C_{t_{0}})}\Bigg)^{\nicefrac{1}{R'}}
\\
\nonumber
&=
\ell \cdot R' \Bigg(\frac{\cost(\pointset, C_{t_{R'}})}{\cost(\pointset, \{c_1\})}\Bigg)^{\nicefrac{1}{R'}}.
\end{align}

We now use this equation to prove items II and III. For item II, we
let random variable $R'$ to be
$$R'= 2e\ceil{k/\ell} +
\log \frac{\cost(\pointset, \{c_1\})}{\OPT_k(\pointset)}.$$
Note that $R'$ depends on the first center $c_1$ (which is chosen in the very beginning of the algorithm) but not on the
Poisson processes $P_t(x)$. Since, $C_t$ always contains at most $k$ centers, we have $\cost(x, C_{t_{R'}})\geq \OPT_k(\pointset)$, and consequently
$$s_{R'}(\pointset) \geq
\ell \cdot R' \Bigg(\frac{\OPT_k(\pointset)}{\cost(\pointset, \{c_1\})}\Bigg)^{\nicefrac{1}{R'}}>
\ell \cdot 2e\ceil{k/\ell} \cdot \nicefrac{1}{e}
\geq 2k.$$
The expected number of jumps of the Poisson process $Q_{s}(\pointset)$ in the interval $[0,s_{R'}(\pointset)]$ equals $Q_{s_R(\pointset)}(\pointset)$. Observe that
$$Q_{s_R(\pointset)}(\pointset)\geq Q_{2k}(\pointset)$$
and $Q_{2k}(\pointset)$ is a Poisson random variable with
parameter $2k$. By the Chernoff bound%
\footnote{We use the bound $\Pr\{P \leq k\}\leq e^{-\lambda}\big(e\lambda/k\big)^k$, where $P$ is a Poisson random variable with parameter $\lambda$ and $k< \lambda$. See e.g., Theorem 5.4.2 in~\citet{MU-Book}.},
it makes fewer than $k$ jumps with exponentially small probability in $k$; namely, with probability at most $(e/2)^{-k}$. Thus, with probability at least $1 - (e/2)^{-k}$, the algorithm selects $k$ centers in the first $R'$ rounds. Moreover, if it does not happen in the first $R^*$ rounds, then it selects $k$ centers by the end of the second $R'$ rounds again with  probability at least $1 - (e/2)^{-k}$ and so on. Hence, the expected number of rounds till it selects $k$ centers is $(1+o_k(1))R'$. Finally, observe that the expectation of $\cost(\pointset, \{c_1\})$ over the choice of the first center equals $2\OPT_k(\pointset)$. Since $\log(\cdot)$ is a convex function, we have
$$\E[R']\leq 2e\ceil{k/\ell} +
\log \frac{2\OPT_1(\pointset)}{\OPT_k(\pointset)}.$$
Therefore, we showed that the expected number of rounds is upper bounded by
the right hand side of the expression above times a multiplicative factor
of $(1+o_k(1))$. A slightly more careful analysis gives a bound of
$$(1+o_k(1))\Bigg(e\ceil{k/\ell} +
\log \frac{2\OPT_1(\pointset)}{\OPT_k(\pointset)}\Bigg).$$
This concludes the proof of item II.
\end{proof}
\begin{proof}[Proof of Part III]
We now prove item III. Denote $T = t_{R^*}$.
Consider the event
$$
\mathcal{E} = \big\{\text{algorithm samples $k$ centers in the first $R^*$ rounds}\big\}.
$$
Let $\bar{\calE}$ be the complimentary events to $\calE$.
Then,
$$
\E\big[\cost(\pointset, C_T)\big]
=
\E\big[\cost(\pointset, C_T) \cdot \one(\calE)\big] +
\E\big[\cost(\pointset, C_T) \cdot \one(\bar{\calE})\big].
$$
We now separately upper bound each of the terms on the right hand side. It is easy to upper bound the first term:
$$\E[\cost(\pointset, C_{T}) \cdot \one(\calE)]\leq 5(\ln k +2)\cdot \OPT_k(\pointset),$$
because the distribution of centers returned by
$k$-means$\kpois$ is identical to the distribution of centers
returned by $k$-means++. We now bound the second term.
Denote by $\calD_{\rho}$ the event
$$
\calD_{\rho} = \bigg\{\cost(\pointset, C_T)\geq \Big(\frac{\rho k}{\ell R^*}\Big)^{R^*} \cost(\pointset, \{c_1\})\bigg\}.
$$

We prove the following claim.
\begin{claim}\label{cl:prob-more-rho}
The following inequality holds for every real number $\rho\in[1,\ell R^*/k]$ and any choice of the first center $c_1$:
$$\Pr\big(\bar\calE \text{ and } \calD_{\rho}\mid c_1\big)
\leq e^{-(\rho -1) k}\rho^{k-1}.$$
\end{claim}
\begin{proof}
We use inequality (\ref{eq:sR}) with
$R'= R^*$:
$$s_{R^*}(\pointset) \geq \ell \cdot R^* \Bigg(\frac{\cost(\pointset, C_T)}{\cost(\pointset, \{c_1\})}\Bigg)^{\nicefrac{1}{R^*}}.$$
It implies that $s_{R^*}(\pointset)\geq \rho k$ if event $\calD_{\rho}$ occurs. On the other hand if $\bar \calE$ occurs, then the number of centers chosen by the end of round $R^*$ is less than $k$ and,
consequently, the number of jumps of $P_t({\pointset})$ in the interval $[0,T]$ is less than $k$:
$$P_T({\pointset}) \equiv Q_{s_{R^*}(\pointset)}(\pointset) < k.$$
Hence, we can bound $\Pr(\bar \calE \text { and } \calD_{\rho}\mid c_1)$ as follows:
\begin{multline*}
\Pr(\bar \calE \text { and } \calD_{\rho}) \leq
\Pr\big(\calD_{\rho} \text{ and } Q_{s_{R^*}}(\pointset)<
k\mid c_1\big) \leq
\\ \leq
\Pr\big(\calD_{\rho} \text{ and } Q_{\rho k}(\pointset)< k
\mid c_1 \big)
\leq
\Pr\big(Q_{\rho k}(\pointset)< k \mid c_1\big).
\end{multline*}
Random variable $Q_{\rho k}(\pointset)$ has the Poisson distribution with parameter $\rho k$ and is independent of $c_1$. By the Chernoff bound, the probability that $Q_{\rho k}(\pointset)\leq k-1$ is at most (as in Part II of the proof):
$$\Pr\big\{Q_{\rho k}(\pointset)\leq k-1\big\}\leq
e^{-\rho k}\Big(\frac{e \rho k}{k-1}\Big)^{k-1} =
e^{-(\rho -1) k-1}\rho^{k-1}\cdot 
\underbrace{\bigg(\frac{k}{k-1}\bigg)^{k-1}}_{\leq e}
\leq e^{-(\rho -1) k}\rho^{k-1}.
$$
This completes the proof of Claim~\ref{cl:prob-more-rho}.
\end{proof}

Let
$$Z=\bigg(\frac{\ell R^*}{k}\bigg)^{R^*}\cdot \frac{\cost(\pointset, C_T)}{\cost(\pointset,\{c_1\} )}.$$
Then, by Claim~\ref{cl:prob-more-rho},
\begin{equation}\label{eq:prob-DZ}
\Pr\big(\bar\calE \text{ and } Z\geq \rho^{R^*} \mid c_1\big)\leq
e^{-(\rho-1)k}\rho^{k-1}. 
\end{equation}
Write,
$$\E\big[\one(\bar \calE) \cdot Z \mid c_1\big] =
\int_{0}^{\infty} \Pr\big(\one(\bar \calE) \text{ and } Z\geq r \mid c_1\big) dr
\leq 1 + \int_{1}^{\infty} \Pr\big(\one(\bar \calE) \text{ and } Z\geq r \mid c_1\big)\,dr.
$$
We now substitute $r=\rho^{R^*}$ and then use~(\ref{eq:prob-DZ}):
\begin{align*}
\E\big[Z \cdot \one(\bar \calE) \mid c_1\big] &\leq 1 +
R^*\int_{1}^{\infty} \Pr\big(\bar\calE \text{ and } Z\geq \rho^{R^*} \mid c_1\big)
\cdot \rho^{R^*-1}
d\rho\\
&\leq 1 + R^* \int_{1}^{\infty}
e^{-(\rho-1)k}\rho^{k + R^* -2} d\rho.
\end{align*}

We note that $R^*< k$, since our algorithm chooses at least one center in each round. Thus, by Lemma~\ref{lem:rho-integral} (which we prove below), the integral on the right hand side is upper bounded by
$\nicefrac{eR^*}{2} \cdot (\nicefrac{4}{e})^{R^*}$.
Hence,
$$\E\big[Z \cdot \one(\bar \calE) \mid c_1\big]\leq 1+
R^* \cdot \bigg(\frac{4}{e}\bigg)^{R^*-2}.
$$
Multiplying both sides of the inequality by $(\nicefrac{k}{\ell R^*})^{R^*}\cdot
{\cost(\pointset,\{c_1\})}$ and taking the expectation over $c_1$, we get the desired inequality:
\begin{align*}
\E\big[\cost(\pointset, C_T) \cdot \one(\bar{\calE})\big] &\leq
\bigg(1 + R^*\; \bigg(\frac{4}{e}\bigg)^{R^*}\bigg)
\Big(\frac{k}{\ell R^*}\Big)^{R^*}
\E_{c_1}\big[\cost(\pointset,\{c_1\}\big]\\
&=\bigg(1 + R^* \; \Big(\frac{4}{e}\Big)^{R^*-2}\bigg)
\Big(\frac{k}{\ell R^*}\Big)^{R^*}
\cdot 2\OPT_1(\pointset)
\\
&<2R^* \; \bigg(\frac{4k}{e\ell R^*}\bigg)^{R^*}\OPT_1(\pointset).
\end{align*}

This finishes the proof of Theorem~\ref{thm:main-poisson-kpp}.
\end{proof}

\begin{lemma}\label{lem:rho-integral}
For $R^*< k$, we have
$$\int_{1}^{\infty}
e^{-(\rho-1)k}\rho^{k + R^* -2} d\rho\leq
\frac{e}{2}\bigg(\frac{4}{e}\bigg)^{R^*}.$$
\end{lemma}
\begin{proof}
Since $e^{-(\rho-1)}\rho\leq 1$ for all $\rho \geq 1$, we have
$e^{-(\rho-1)k}\rho^{k} \leq e^{-(\rho-1)R^*}\rho^{R^*}$ for any $R^* < k$.
Thus, we have
\begin{align*}
\int_{1}^{\infty}
e^{-(\rho-1)k}\rho^{k + R^* -2} d\rho &\leq
\int_{1}^{\infty}
e^{-(\rho-1)R^*}\rho^{2R^* -3} d\rho
= e^{R^*}\int_{1}^{\infty}
e^{-\rho R^*}\rho^{2R^* - 3} d\rho\\
&= e^{R^*}\int_{1}^{\infty}
(e^{-\rho}\rho^{2})^{R^*} \rho^{-3} d\rho.
\end{align*}
Observe that $e^{-\rho}\rho^2 \leq 4/e^2$ for any $\rho \geq 1$. Hence,
$$(e^{-\rho}\rho^{2})^{R^*}
=
(e^{-\rho}\rho^{2})^{R^*-1}\cdot e^{-\rho}\rho^{2}
\leq (4/e^2)^{R^*-1}e^{-\rho}\rho^{2}.$$
Thus,
$$\int_{1}^{\infty}
e^{-(\rho-1)k}\rho^{k + R^* -2} d\rho \leq
\frac{4^{R^*-1}\cdot e^{R^*}}{e^{2(R^*-1)}}\cdot
\int_1^{\infty}\frac{e^{-\rho}}{\rho} \; d\rho =
\frac{4^{R^*-1}}{e^{R^*-2}}
\cdot \frac{1}{4} = \bigg(\frac{4}{e}\bigg)^{R^*-2}.
$$
\end{proof}

\subsection{\texorpdfstring{Lazy implementation of $k$-means$\kpois$}{Lazy implementation of Exponential Race k-means++}}\label{sec:lazy}

We now describe how we can efficiently implement
the $k$-means$\kpois$ algorithm using a lazy reservoir sampling.  We remind the reader that the time of the first jump of a Poisson process with parameter $\lambda$ is distributed as the exponential distribution with parameter $\lambda$. Imagine for a moment, that the arrival rates of our Poisson processes were constant. Then, in order to select the first $k$ jumps, we would generate independent exponential random variables with parameters $\lambda(x)$ for all $x$ and choose $k$ smallest values among them. This algorithm is known as the reservoir sampling(see~\citet{ReservoirSampling}). To adapt this algorithm to our needs, we need to update the arrival rates of the exponential random variables. Loosely speaking, we do so by generating exponential random variables with rate $1$ for Poisson processes $Q_s(x)$ which are described above and then updating the speeds $\lambda_t(x)$ of variables $s_t(x)$. We now formally describe the algorithm.

In the beginning of every round $i$, we recompute costs of all points in the data set. Then, we draw an independent exponential random variable $\mathcal{S}_x$ with rate $1$ for every point $x$, and let $S_t(x)=\mathcal{S}_x$ . We set
$$\tau_t(x)=\frac{S_t(x)}{\lambda_t(x)}.$$
Think of $S_t(x)$ as the distance $s_t(x)$ needs to travel till process $Q_s(x)$ jumps; $\lambda_t(x)$ is the speed of point~$s_t(x)$; and $\tau_t(x)$ is the time left till $Q_s(x)=P_t(x)$
jumps if the speed $\lambda_t$ does not change. Among all points $x\in X$, we select a tentative set of centers $Z$ for this round. The set $Z$ contains all points $x$
with $t_{i-1} + \tau_t(x)\leq t_{i}$. This is the set of all points for which their Poisson processes would jump in the current round if their arrival rates remained the same till the end of the round. Since the arrival rates can only decrease in our algorithm, we  know for sure that for points $x$ outside of $Z$, the corresponding processes $P_t(x)$ will not jump in this round. Thus, we can safely ignore those points during the current round.

We also note that in the unlikely event that the initial set $Z$ is empty, we choose $x$ with the smallest time $\tau_t(x)$ and add it to the set of centers $C_t$. (This is equivalent to choosing a point with probability proportional to $\cost(x,C_t)$ by the memorylessness property of the exponential distribution).

The steps we described above -- updating costs $\cost(x,C_t)$, drawing exponential random variables $\mathcal{S}_x$, and selecting points in the set $Z$ -- can be performed in parallel using one pass over the data set. In the rest of the current round, our algorithm deals only with the set $Z$ whose size in expectation is at most $\ell$ (see below).

While the set $Z$ is not empty we do the following. We choose $x\in Z$ with the smallest value of $\tau_t(x)$. This $x$ corresponds to the process that jumps first.
Then, we perform the following updates: We add $x$ to the set of centers $C_t$. We set the ``current time'' $t$ to $t = t' + \tau_{t'}(x)$, where $t'$ is the time of the previous update. If $x$ is the first center selected in the current round, then we let $t'$ to be the time when the round started (i.e., $t_{i-1}$). We recompute the arrival rates (speeds) $\lambda_t(x)$ for each $x$ in $Z$. Finally, we update the values of all $\tau_t(x)$ for $x\in Z$ using the formula
$$\tau_t(x) =
\frac{S_t(x)-\lambda_{t'}(x)\cdot(t-t')}{\lambda_t(x)},$$
here $\lambda_{t'}(x)\cdot(t-t')$ is the distance variable $s_t(x)$ moved from the position where it was at time~$t'$; $S_t(x)-\lambda_{t'}(x)\cdot(t-t')$ is the remaining distance $s_t(x)$ needs to travel till the process $Q_t(x)$ jumps; and $\tau_t(x)$ is the remaining
time till $P_t(x)$ jumps if we do not update its arrival rate. After we update $\tau_t(x)$, we prune the set $Z$. Specifically, we remove from set $Z$ all points $x$ with
$t+\tau_t(x) > t_{i}$. As before, we know for sure that if $x$ is removed from $Z$, then the corresponding  processes $P_t(x)$ will not jump in the current round.

This algorithm simulates the process we described in the previous section. The key observation is that Poisson processes $P_t(x)$ we associate with points $x$ removed from $Z$ cannot jump
in this round and thus can be safely removed from our consideration. We now show that the expected size of the set $Z$ is at most $\ell$. In the next section, we analyze the running time of this algorithm.

\medskip

Then we show that the expected size of the set $Z$ in the beginning of each round $i+1$ is at most $\ell$. Since every point $x$ belongs to $Z$ with probability
$$\Pr\{x\in Z\} =
\Pr\bigg\{\ \frac{\mathcal{S}_x}{\cost(x,C_{t_i})}
\leq \frac{\ell}{\cost(\pointset, C_{t_i})} \bigg\} =
\Pr\bigg\{\ \mathcal{S}_x
\leq \ell \cdot \frac{\cost(x,C_{t_i})}{\cost(\pointset, C_{t_i})}\bigg\}.$$
The right hand side is the probability that the Poisson  process $Q_s(x)$ with rate 1 jumps in the interval
of length $\ell \cdot \cost(x,C_{t_i})/\cost(\pointset, C_{t_i})$ which is upper bounded by the expected number of jumps of $Q_s(x)$ in this interval. The expected number of jumps exactly equals $\ell \cdot \cost(x,C_{t_i})/\cost(\pointset, C_{t_i})$. Thus, the expected size of $Z$ is upper bounded as
$$\E|Z| = \sum_{z\in \pointset}
\Pr\{z\in Z\} \leq \sum_{z \in \pointset} \ell \cdot \frac{\cost(z,C_{t_i})}{\cost(\pointset, C_{t_i})} = \ell.$$
\subsection{Run time analysis}
According to our analysis above, the number of new centers chosen at each round of $k$-means$\kpois$ is at most the size of set $Z$, which is $O(\ell)$ with high probability. In the beginning of every round, we need to update costs of all data points, which requires $O(n\ell d)$ time. In each round, we also need to maintain the rates of all points in set $Z$, which needs $O(\ell^2 d)$ time. Thus, the total running time for $k$-means$\kpois$ with $R$ rounds is $O(Rn\ell d)$. We note that before running our algorithm, we can reduce the dimension $d$ of the space to $O(\log k)$ using the Johnson–Lindenstrauss transform (see~\citet{JLpaper}). This will increase the approximation factor by a factor of $(1+\varepsilon)$ but make the algorithm considerably faster (see~\citet{MMR19},
\citet{becchetti2019oblivious}, and \citet{boutsidis2010random}).

\bibliographystyle{abbrvnat}
\bibliography{references}

\newpage

\appendix
\section*{\centering{Appendix}}

In this appendix, we present our experiments, give proofs omitted in the main part of the paper, and provide complimentary lower bounds.

\section{Experiments}\label{sec:experiments}
In this section, we present plots that show that the performance of $k$-means$\parallel$
and ``$k$-means++ with oversampling and pruning'' algorithms are very similar in practice.
Below, we compare the following algorithms on the datasets BioTest from KDD Cup 2004 \cite{kddcup2004} and COVTYPE from the UCI ML repository \cite{Dua:2019}:
\begin{itemize}
  \item Regular $k$-means++. The performance of this algorithm is shown with a solid black line on the plots below.
  \item $k$-means$\parallel$ without pruning. This algorithm samples $k$ centers using $k$-means$\parallel$ with $T = 5$ rounds and $\ell = k/T$.
  \item $k$-means$\parallel$. This algorithm first samples $5k$ centers using $k$-means$\parallel$ and then subsamples $k$ centers using $k$-means++.
  The performance of this algorithm is shown with a dashed blue line on the plots below.
  \item $k$-means++ with oversampling and pruning. This algorithm first samples $5k$ centers using $k$-means++ and then subsamples $k$ centers using $k$-means++. The performance of this algorithm is shown with a thin red line on the plots below.
\end{itemize}

For each $k=5,10,\cdots, 200$, we ran these algorithms for 50 iterations and took their average. We normalized all costs by dividing them by the cost of $k$-means++ with $k=1000$ centers.


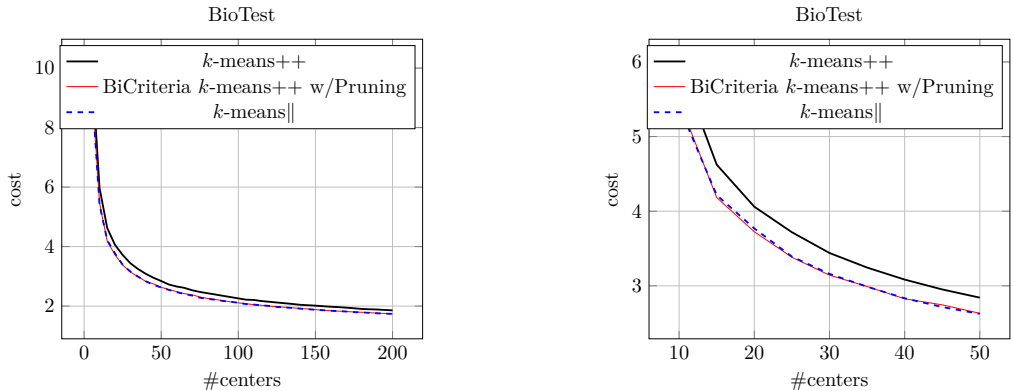
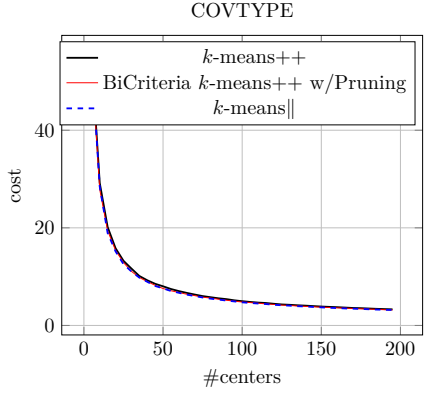
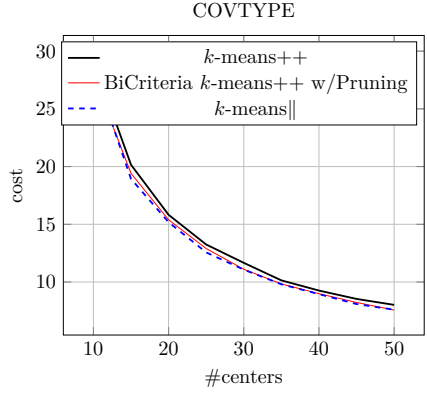
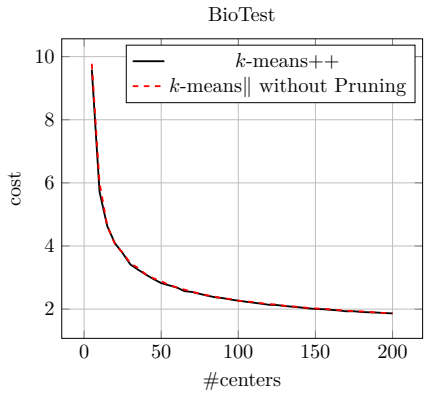
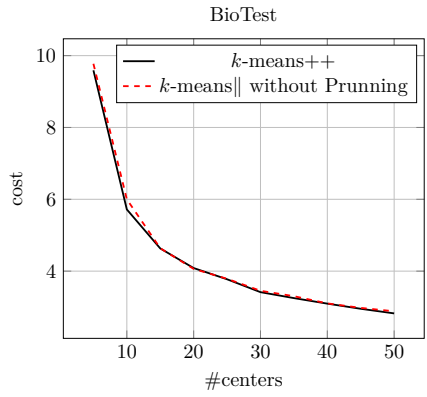
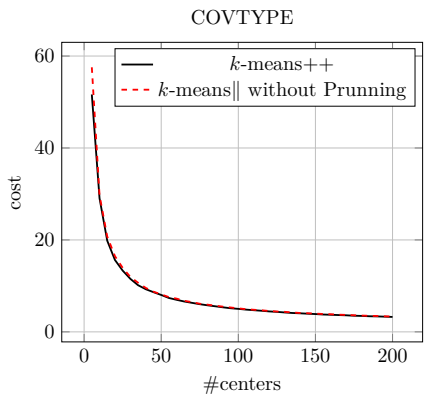
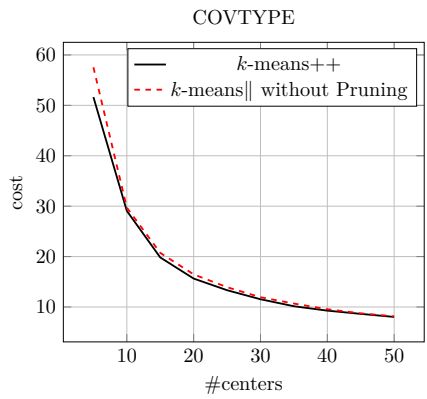
\begin{figure}[h]
    \begin{minipage}{0.45\linewidth}
        \begin{tikzpicture}[scale=0.7]
        \begin{axis}[
        title= {BioTest},
        xlabel={\#centers},
        ylabel={cost},
        grid = major]
        \addplot[black,domain=1:2, line width = 1pt]  table[x=centers,y=avgKMeansPP,col sep=comma] {plotdata/bio-test-results.csv}; \addlegendentry{$k$-means++}
        \addplot[red,domain=1:2, line width = 0.5pt]  table[x=centers,y=avgBicriteria,col sep=comma] {plotdata/bio-test-results.csv}; \addlegendentry{BiCriteria $k$-means++ w/Pruning}
        \addplot [blue, domain=1:2, dashed, line width = 1pt] table[x=centers,y=avgKMeansParallel,col sep=comma] {plotdata/bio-test-results.csv}; \addlegendentry{$k$-means$\parallel$}
        \end{axis}
        \end{tikzpicture}
        \captionsetup{justification=centering}
    \end{minipage}
    \quad
    \begin{minipage}{0.45\linewidth}
        \begin{tikzpicture}[scale=0.7]
        \begin{axis}[
        title= {BioTest},
        xlabel={\#centers},
        ylabel={cost},
        grid = major]
        \addplot[black,domain=1:2, line width = 1pt]  table[x=centers,y=avgKMeansPP,col sep=comma] {plotdata/bio-test10-50.csv}; \addlegendentry{$k$-means++}
        \addplot[red,domain=1:2, line width = 0.5pt]  table[x=centers,y=avgBicriteria,col sep=comma] {plotdata/bio-test10-50.csv}; \addlegendentry{BiCriteria $k$-means++ w/Pruning}
        \addplot [blue, domain=1:2, dashed, line width = 1pt] table[x=centers,y=avgKMeansParallel,col sep=comma] {plotdata/bio-test10-50.csv}; \addlegendentry{$k$-means$\parallel$}
        \end{axis}
        \end{tikzpicture}
        \captionsetup{justification=centering}
    \end{minipage}
\end{figure}

\begin{figure}
    \begin{minipage}{0.45\linewidth}
        \begin{tikzpicture}[scale=0.7]
        \begin{axis}[
        title= {COVTYPE},
        xlabel={\#centers},
        ylabel={cost},
        grid = major]
        \addplot[black,domain=1:2, line width = 1pt]  table[x=centers,y=avgKMeansPP,col sep=comma] {plotdata/covtype.csv}; \addlegendentry{$k$-means++}
        \addplot[red,domain=1:2, line width = 0.5pt]  table[x=centers,y=avgBicriteria,col sep=comma] {plotdata/covtype.csv}; \addlegendentry{BiCriteria $k$-means++ w/Pruning}
        \addplot [blue, domain=1:2, dashed, line width = 1pt] table[x=centers,y=avgKMeansParallel,col sep=comma] {plotdata/covtype.csv};\addlegendentry{$k$-means$\parallel$}
        \end{axis}
        \end{tikzpicture}
        \captionsetup{justification=centering}
    \end{minipage}
    \quad
    \begin{minipage}{0.45\linewidth}
        \begin{tikzpicture}[scale=0.7]
        \begin{axis}[
        title= {COVTYPE},
        xlabel={\#centers},
        ylabel={cost},
        grid = major]
        \addplot[black,domain=1:2, line width = 1pt]  table[x=centers,y=avgKMeansPP,col sep=comma] {plotdata/covtype10-50.csv}; \addlegendentry{$k$-means++}
        \addplot[red,domain=1:2, line width = 0.5pt]  table[x=centers,y=avgBicriteria,col sep=comma] {plotdata/covtype10-50.csv}; \addlegendentry{BiCriteria $k$-means++ w/Pruning}
        \addplot [blue, domain=1:2, dashed, line width = 1pt] table[x=centers,y=avgKMeansParallel,col sep=comma] {plotdata/covtype10-50.csv};\addlegendentry{$k$-means$\parallel$}
        \end{axis}
        \end{tikzpicture}
        \captionsetup{justification=centering}
    \end{minipage}
\end{figure}


\begin{figure}
    \begin{minipage}{0.45\linewidth}
        \begin{tikzpicture}[scale=0.7]
        \begin{axis}[
        title= {BioTest},
        xlabel={\#centers},
        ylabel={cost},
        grid = major]
        \addplot[black,domain=1:2, line width = 1pt]  table[x=centers,y=avgKMeansPP,col sep=comma] {plotdata/biotest_results_kpar_kpp.csv}; \addlegendentry{$k$-means++}
        \addplot [red, domain=1:2, dashed, line width = 1pt] table[x=centers,y=avgKMeansParallelForK,col sep=comma] {plotdata/biotest_results_kpar_kpp.csv};\addlegendentry{$k$-means$\parallel$ without Pruning}
        \end{axis}
        \end{tikzpicture}
        \captionsetup{justification=centering}
    \end{minipage}
    \quad
    \begin{minipage}{0.45\linewidth}
        \begin{tikzpicture}[scale=0.7]
        \begin{axis}[
        title= {BioTest},
        xlabel={\#centers},
        ylabel={cost},
        grid = major]
        \addplot[black,domain=1:2, line width = 1pt]  table[x=centers,y=avgKMeansPP,col sep=comma] {plotdata/biotest_results_kpar_kpp_50.csv}; \addlegendentry{$k$-means++}
        \addplot [red, domain=1:2, dashed, line width = 1pt] table[x=centers,y=avgKMeansParallelForK,col sep=comma] {plotdata/biotest_results_kpar_kpp_50.csv};\addlegendentry{$k$-means$\parallel$ without Prunning}
        \end{axis}
        \end{tikzpicture}
        \captionsetup{justification=centering}
    \end{minipage}
\end{figure}

\begin{figure}
    \begin{minipage}{0.45\linewidth}
        \begin{tikzpicture}[scale=0.7]
        \begin{axis}[
        title= {COVTYPE},
        xlabel={\#centers},
        ylabel={cost},
        grid = major]
        \addplot[black,domain=1:2, line width = 1pt]  table[x=centers,y=avgKMeansPP,col sep=comma] {plotdata/covtype_kpar_kpp.csv}; \addlegendentry{$k$-means++}
        \addplot [red, domain=1:2, dashed, line width = 1pt] table[x=centers,y=avgKMeansParallelForK,col sep=comma] {plotdata/covtype_kpar_kpp.csv};\addlegendentry{$k$-means$\parallel$ without Prunning}
        \end{axis}
        \end{tikzpicture}
        \captionsetup{justification=centering}
    \end{minipage}
    \quad
    \begin{minipage}{0.45\linewidth}
        \begin{tikzpicture}[scale=0.7]
        \begin{axis}[
        title= {COVTYPE},
        xlabel={\#centers},
        ylabel={cost},
        grid = major]
        \addplot[black,domain=1:2, line width = 1pt]  table[x=centers,y=avgKMeansPP,col sep=comma] {plotdata/covtype_kpar_kpp_50.csv}; \addlegendentry{$k$-means++}
        \addplot [red, domain=1:2, dashed, line width = 1pt] table[x=centers,y=avgKMeansParallelForK,col sep=comma] {plotdata/covtype_kpar_kpp_50.csv};\addlegendentry{$k$-means$\parallel$ without Pruning}
        \end{axis}
        \end{tikzpicture}
        \captionsetup{justification=centering}
    \end{minipage}
\end{figure}

\newpage

\section{Details for Preliminaries}\label{sec:prelim_details}
For any set of points $\Y \subset \R^d$, let $\mu = \sum_{x \in \Y} x/\abs{\Y}$ be the \textit{centroid} of the cluster $\Y$. Then, the optimal cost of $\Y$ with one center,
\begin{align*}
    \opt_1(\Y) = \sum_{x\in \Y} \norm{x-\mu}^2 = \frac{\sum_{(x,y)\in\Y\times\Y} \norm{x-y}^2}{2\abs{\Y}}.
\end{align*}
This is a well known formula which is often used for analyzing of $k$-means algorithms. For completeness, we give a proof below.
\begin{proof}
Consider any point $z \in \R^d$, then we have:
\begin{align*}
    \cost(\Y,\{z\}) &= \sum_{x\in \Y} \norm{x-z}^2 =  \sum_{x\in \Y} \norm{(x-\mu) + (\mu - z)}^2\\
    &= \sum_{x\in \Y} \rbr{\norm{x-\mu}^2 + \norm{\mu-z}^2 + 2\ip{x-\mu, \mu-z}}
    \\&= \sum_{x\in \Y} \norm{x-\mu}^2 + \abs{\Y} \cdot \norm{\mu-z}^2 + 2 \ip{\sum_{x\in \Y} (x-\mu), \mu-z}
    \\&= \sum_{x\in \Y} \norm{x-\mu}^2 + \abs{\Y} \cdot \norm{\mu-z}^2.
\end{align*}
Thus, the optimal choice of $z$ to minimize $\cost(\Y,\{z\})$ is $\mu$ and $\opt_1(\Y) = \sum_{x\in \Y} \norm{x-\mu}^2$.

\begin{align*}
    \sum_{x\in \Y} \norm{x-\mu}^2 &= \sum_{x\in \Y} \ip{x-\mu, x-\mu}
    = \sum_{x\in \Y} \ip{x, x-\mu}
    \\&= \sum_{x\in \Y} \ip{x, x - \sum_{y \in \Y} \frac{y}{\abs{\Y}}} =\frac{1}{\abs{\Y}} \sum_{(x,y)\in\Y\times\Y} \ip{x, x - y}
    \\&= \frac{1}{2\abs{\Y}} \rbr{\sum_{(x,y)\in\Y\times\Y} \ip{x, x - y} + \sum_{(x,y)\in\Y\times\Y} \ip{y, y - x}}
    \\&= \frac{\sum_{(x,y)\in\Y\times\Y} \norm{x-y}^2}{2\abs{\Y}}.
\end{align*}
\end{proof}

\section{Lower bounds}\label{sec:lb}
\subsection{Lower bound on the cost of covered clusters}
We show the following lower bound on the expected cost of a covered cluster in $k$-means++. Therefore, the $5$-approximation in Lemma~\ref{lem:5OPT} is tight.
\begin{theorem}\label{thm:5-approx-tight}
For any $\varepsilon > 0$, there exists an instance of $k$-means such that for a set $\cluster \in \pointset$ and a set of centers $C \in \real^\dimension$, if a new center $c$ is sampled from $\cluster$ with probability $\Pr(c = x) = \cost(x, C)/\cost(P,C)$, then
$$
\E_c\sbr{\cost(P,C\cup\cbr{c})} \geq (5-\varepsilon) \opt_1(P).
$$
\end{theorem}

\begin{proof}
Consider the following one dimensional example, where $\cluster$ contains $t$ points at $0$ and one point at $1$, and the closest center already chosen in $C$ to $\cluster$ is at $-1$.

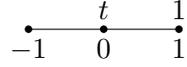
\begin{figure}[ht]
\centering
\begin{tikzpicture}
    \draw (-1,0) -- (1,0);
    \fill[black] (-1,0) circle (0.5 mm) node[below] {$-1$};
    \fill[black] (0,0) circle (0.5 mm) node[below] {$0$} node[above] {$t$};
    \fill[black] (1,0) circle (0.5 mm) node[below] {$1$} node[above] {$1$};
  \end{tikzpicture}
\end{figure}

The new center $c$ will be chosen at $0$ with probability $\frac{t}{t+4}$, and at $1$ with probability $\frac{4}{t+4}$. Then, the expected cost of $\cluster$ is
$$
\E_c\sbr{\cost(P,C\cup\cbr{c})} = 1\cdot\frac{t}{t+4} + t\cdot\frac{4}{t+4} = \frac{5t}{t + 4};
$$
and the optimal cost of $\cluster$ is $\opt_1(\cluster) \leq 1$. Thus, by choosing $t \geq 4(5-\varepsilon)/\varepsilon$, we have
$$
\E_c\sbr{\cost(P,C\cup\cbr{c})} \geq (5-\varepsilon) \opt_1(P).
$$
\end{proof}

\subsection{Lower bound on the bi-criteria approximation}

In this section, we show that the bi-criteria approximation
bound of $O(\ln \frac{k}{\extracenters})$ is tight up to constant factor. Our proof follows the approach by~\citet{brunsch2013bad}. We show the following theorem.

\begin{theorem}~\label{thm:lowerbound_bi}
For every $k>1$ and $\extracenters \leq k$, there exists an instance $\pointset$ of $k$-means such that the bi-criteria $k$-means++ algorithm with $k+\extracenters$ centers returns a solution of cost greater than
$$\frac{1}{8}\log \frac{k}{\extracenters} \cdot \opt_k(\pointset)$$ with probability at least
$1 - e^{-\sqrt{k}/2}$.
\end{theorem}

\textbf{Remark:} This implies that the expected cost of bi-criteria $k$-means with $k+\Delta$ centers is at least $$\frac{1-e^{-\sqrt{k}/2}}{8}\cdot \log \frac{k}{\extracenters} \cdot \opt_k(\pointset).$$

\begin{proof}
For every $k$ and $\extracenters \geq \sqrt{k}$, we consider the following instance. The first cluster is a scaled version of the standard simplex with $N \gg k$ vertices centered at the origin, which is called the heavy cluster. The length of the edges in this simplex is $1/\sqrt{N-1}$. Each of the remaining $k-1$ clusters contains a single point on $k-1$ axes, which are called light clusters. These clusters are located at distance $\sqrt{\alpha}$ from the center of the heavy cluster and $\sqrt{2\alpha}$ from each other, where $\alpha = \frac{\ln(k/\extracenters)}{4\extracenters}$.

For the sake of analysis, let us run $k$-means++ till we cover all  clusters. At the first step, the $k$-means++ algorithm almost certainly selects a
center from the heavy cluster since $N \gg k$. Then, at each step, the algorithm can select a center either from one of uncovered light clusters or from the
heavy cluster. In the former case, we say that the algorithm hits a light cluster,
and in the latter case we say that the algorithm misses a light cluster. Below,
we show that with high probability the algorithm makes at least $2\extracenters$ misses
before it covers all but $\Delta$ light clusters.

\begin{lemma}~\label{lem:lb_miss}
Let $\extracenters\geq \sqrt{k}$. By the time the $k$-means++ algorithm covers all but $\extracenters$ light clusters, it makes greater than $2\extracenters$ misses with probability at least $1-e^{-\sqrt{k}/2}$.
\end{lemma}
\begin{proof}[Proof sketch]
Let $\varepsilon = 1/\sqrt{N}$. Observe that $k$-means++ almost certainly covers all clusters in $\varepsilon N$ steps (since $N\gg k$). So in the rest of this proof sketch, we assume that the
number chosen centers is at most $\varepsilon N$ and, consequently, at least $(1-\varepsilon)N$ points in the heavy cluster are not selected as centers. Hence, the cost of the heavy cluster is at least $1-\varepsilon$.

Consider a step of the algorithm when exactly $u$ light clusters remain uncovered. At this
step, the total cost of all light clusters is $\alpha u$ (we assume for simplicity that distance between the light clusters and the closest chosen center in the heavy cluster is the same as the distance to the origin). The cost of the heavy cluster is at least $1-\varepsilon$. The probability that the algorithm chooses a center from the heavy cluster and thus misses a light cluster is at least $(1-\varepsilon)/(1+\alpha u)$.

Define random variables $\cbr{X_u}$ as follows. Let $X_u = 1$ if the algorithm misses a
cluster at least once when the number of uncovered light clusters is $u$; and let
$X_u = 0$, otherwise. Then, $\cbr{X_u}$ are independent Bernoulli random variables.
For each $u$, we have $\prob{X_u = 1} \geq (1-\varepsilon)/(1+\alpha u)$.

Observe that the total number of misses is lower bounded by $\sum_{u = \extracenters}^{k-1} X_u$. Then, we have
\begin{align*}
\E\sbr{\sum_{u=\extracenters}^{k-1} X_u } &\geq (1-\varepsilon)\sum_{u = \extracenters}^{k-1} \frac{1}{1+\alpha u} \geq (1-\varepsilon)\int_{\extracenters}^k \frac{\de u}{1+\alpha u} \\
&= (1-\varepsilon)\alpha^{-1}\ln {\frac{1+\alpha k}{1+\alpha\extracenters}}  \\
&\geq (1-\varepsilon)\alpha^{-1} \ln \frac{k}{\extracenters} = 4(1-\varepsilon)\extracenters.
\end{align*}
Let $\mu = \E\sbr{\sum_{u=\extracenters}^{k-1} X_u } \geq 4(1-\varepsilon)\extracenters$. By the Chernoff bound for Bernoulli random variables, we have
$$
\prob{\sum_{u = \extracenters}^k X_u \leq 2\extracenters } \leq e^{-\mu} \rbr{\frac{e\mu}{2\extracenters}}^{2\extracenters}.
$$
Since $f(x) = e^{-x}(\frac{ex}{2\extracenters})^{2\extracenters}$ is a monotone decreasing function for $x \geq 2\extracenters$, we have
$$
\prob{\sum_{u = \extracenters}^k X_u \leq 2\extracenters }  \leq e^{-(2-4\varepsilon)\extracenters}\cdot 2^{2\extracenters} \leq e^{-\extracenters/2}.
$$
Hence, with probability as least $1-e^{-\sqrt{k}/2}$, the number of misses is greater than $2\extracenters$.
\end{proof}

For every $k$ and $\extracenters \geq \sqrt{k}$, consider the instance we constructed. By Lemma~\ref{lem:lb_miss}, the algorithm chooses more than $k+\extracenters$ centers to cover all but $\extracenters$ light clusters with probability at least $1-e^{-\sqrt{k}/2}$. Thus, at the time when the algorithm chose $k+\extracenters$ centers, the number of uncovered light clusters was greater than $\extracenters$. Hence, in the clustering with $k+\extracenters$ centers sampled by $k$-means++, the total cost is at least  $\frac{1}{4}\ln\rbr{k/\extracenters}$,
while the cost of the optimal solution with $k$ clusters is $1$. For every $k$ and $\extracenters < \sqrt{k}$, the total cost is at least $\frac{1}{4}\ln(k/\extracenters')$ with $\extracenters' = \sqrt{k}$ extra centers, which concludes the proof.
\end{proof}
\end{document}